\documentclass{article}
\newif\ifdraft
\drafttrue % comment out to hide acomments

\newif\ifarxiv
\arxivfalse
\usepackage{stackengine}
\usepackage[nonatbib, preprint]{neurips_2024}
\usepackage{graphicx}
\usepackage{amsmath, amsthm,amssymb,mathtools}
\usepackage[colorlinks=true,linkcolor=black,citecolor=MidnightBlue,urlcolor=MidnightBlue,unicode]{hyperref}
\usepackage{xspace}
\usepackage{tikz}
\definecolor{MidnightBlue}{RGB}{25,25,112}
\usepackage{booktabs} 
\usepackage{biblatex}
\usepackage{cleveref}
\usepackage{nicefrac}

\newtheorem{theorem}{Theorem}
\newtheorem{definition}{Definition}
\newtheorem{lemma}{Lemma}
\newtheorem{corollary}{Corollary}
\theoremstyle{remark}
\newtheorem*{remark}{Remark}

\newcommand{\Var}[1]{\operatorname{Var}\left[#1\right]}%G:changed mathbf to operatorname in the new command, operators should not be boldface
\newcommand{\Hz}[1]{H_{#1}}%G:command for the events notation
\renewcommand{\epsilon}{\varepsilon}

\newcommand{\futurenote}[1]{\ifarxiv\ifdraft{\color{blue} FG: #1}\fi\fi}

\title{On the Sparsity of the Strong Lottery Ticket Hypothesis}
\author{Emanuele Natale \\ Université Côte d'Azur, \\ CNRS, Inria, I3S, France 
\And Davide Ferré \\ Université Côte d'Azur, \\ CNRS, Inria, I3S, France 
\And  Giordano Giambartolomei \\ Department of Informatics, \\ King’s College London
\And Frédéric Giroire \\ Université Côte d'Azur, CNRS, \\ Inria, I3S, France 
\And Frederik Mallmann-Trenn \\ Department of Informatics, \\ King’s College London
}

\date{April 2024}

\allowdisplaybreaks

\begin{document}

% MACROS
\def\Y{Y}
\def\indexn{n}
\def\Sun{\Sigma^{\mathcal U_n}_{[n]}}
\def\Zun{Z_n^u}
\def\Irv{I_n}
\def\f{f}
\def\fz{f_z}
\def\Fn{F_n}
\def\PL{P_L(n)}
\def\nitbf{\noindent\textbf}
\def\Cmin{c_l}
\def\Cminp{c'_l}
\def\Cmax{c_u}
\def\Cmaxp{c'_u}
\def\Cmaxpp{c''_u}
\def\variance{\mathrm{Var}}
% Macro to choose between fixed-size or sparse
\def\fixedsize{fixed-size\xspace}
% Macro to choose between RFSS or SRSS
\def\RFSS{RFSS}
\def\RFSSlong{Random Fixed-Size Subset Sum}
\def\almostGaussian{sum-bounded}
\def\AlmostGaussian{Sum-Bounded}
\def\gammastar{\gamma'}
\maketitle
\begin{abstract}
    Considerable research efforts have recently been made to show that a random neural network $N$ contains subnetworks capable of accurately approximating any given neural network that is sufficiently smaller than $N$, without any training. 
    This line of research, known as the Strong Lottery Ticket Hypothesis (SLTH), was originally motivated by the weaker Lottery Ticket Hypothesis, which states that a sufficiently large random neural network $N$ contains \emph{sparse} subnetworks that can be trained efficiently to achieve performance comparable to that of training the entire network $N$.
    Despite its original motivation, results on the SLTH have so far not provided any guarantee on the size of subnetworks.
    Such limitation is due to the nature of the main technical tool leveraged by these results, the Random Subset Sum (RSS) Problem.
    Informally, the RSS Problem asks how large a random i.i.d. sample $\Omega$ should be so that we are able to approximate any number in $[-1,1]$, up to an error of $ \epsilon$, as the sum of a suitable subset of $\Omega$. 

    We provide the first proof of the SLTH in classical settings, such as dense and equivariant networks, with guarantees on the sparsity of the subnetworks. Central to our results, is the proof of an essentially tight bound on the Random Fixed-Size Subset Sum Problem (\RFSS{}), a variant of the RSS Problem in which we only ask for subsets of a given size, which is of independent interest.
\end{abstract}

\section{Introduction}
\label{sec:intro}

% \ema{remember to be consistent regarding the choice of sparse vs fixed-size and that we have 9 pages for the main body before refs}
% \ema{Important: there is also a form that should be included in the submission. Davide originally included it in the template and I deleted it because I thought that we had to fill it when submitting, but this is not the case. Please check that we don't forget to add some mandatory things.}

% \fnote{In the past we had to upload the appendix / supplementary material separately.
% The way we did it was to just have the full paper there (not just the appendix). This also means that the main paper might end up with broken references!}
% \dnote{appendix has to be uploaded with the main paper, not separately}\fnote{Cool, are you able to split a pdf easily? }

The Lottery Ticket Hypothesis (LTH) is a research direction that has attracted considerable attention over the years, stemming from the empirical contrast between the fact that, while large neural networks can be successfully trained to achieve good performance on a given task and successively pruned to a great level of sparsity without compromising their performance, researchers have struggled to train sparse neural networks from scratch. 
The authors of \cite{frankleLotteryTicketHypothesis2018} observed that, using a simple pruning strategy (namely Iterative Magnitude Pruning while rewinding the original weights of the remaining edges to their value at initialization), \emph{starting from a sufficiently large random neural networks, it is possible to identify sparse subnetworks that can be trained to achieve the performance achievable by the starting network} (see Figure~\ref{fig:LTH} in the appendix for an illustration).
The previous statement, namely the LTH, soon gave rise to an even stronger one, corroborated by empirical works \cite{zhouDeconstructingLotteryTickets2019,ramanujanWhatHiddenRandomly2020} which proposed ``training-by-pruning" algorithms (see Section~\ref{sec:related} for details), providing evidence that \emph{starting from a sufficiently large random neural networks, it is possible to identify sparse subnetworks that exhibit good performance as they are, without changing the original weights} (see Figure~\ref{fig:truning} in the appendix for an illustration).
By removing the need to analyze the dynamics of training, the last statement, namely the Strong Lottery Ticket Hypothesis (SLTH), allowed a fruitful series of rigorous proofs for increasingly more general architectures (see Section \ref{sec:related} for an overview). 
Such rigorous results can informally be stated as follows: 
\begin{theorem}[Informal statement of previous SLTH results]\label{thm:informal}
    With high probability, a random artificial neural network $N_\Omega$ with $m$ parameters can be pruned so that the resulting subnetwork $N_S$ $\epsilon$-approximates (i.e., approximates up to an error $\varepsilon$) any target artificial neural network $N_t$ with $O\left(m/\log_2(1/\varepsilon)\right)$ parameters.
\end{theorem}
It is important to note that, to this day, we only have proofs on the existence of such subnetworks, also called winning tickets, but it remains an open question how to find them reliably. 

\emph{All theoretical results on the SLTH however have so far not investigated the interplay between the sparsity of the winning ticket $N_S$ and the size of the random neural network $N_\Omega$.}
This is in contrast to the original motivation of the LTH and to the practical application of the aforementioned training-by-pruning algorithms that motivated the SLTH, such as \cite{isikSparseRandomNetworks2022,isikAdaptiveCompressionFederated2024}.
In fact, to approximate target networks with $O\left(m/\log_2(1/\varepsilon)\right)$ parameters, essentially all winning tickets $N_S$ have $\Theta(m)$ parameters (see Appendix~\ref{apx:pensia_lower}), thus being roughly of the same size of the original network $N_\Omega$.
We thus ask the following natural question:
\begin{quote}
    If we want to $\epsilon$-approximate a family of target artificial neural networks with $m_t$ parameters by pruning a fraction $\alpha$, called sparsity, of the $m$ parameters of a random artificial neural network $N_\Omega$, how big should $m$ be?
\end{quote}
% \begin{quote}
%     What is the number of parameters of the family of target artificial neural networks that can be approximated by pruning a fraction $\alpha$ of the $m$ parameters of a random artificial neural network $N_\Omega$?
% \end{quote}
We are particularly interested in the regime in which the density parameter $\gamma = 1 - \alpha$ vanishes as the size of the network increases, so that the size of the winning ticket $N_S$ is $\gamma m = o(m)$.
%\fnote{ we should probably also define density and sparsity somewhere}

The above question has so far remained unanswered as a consequence of the limitation inherited from the core technical tool that has been leveraged so far to prove SLTH results, namely the Random Subset Sum (RSS) Problem \cite{Lueker98}.
Informally, the RSS asks how large a random i.i.d. sample $\Omega$ should be so that we are able to approximate any number in $[-1,1]$ as the sum of a suitable subset of $\Omega$. 
The  applicability of RSS to the SLTH was first recognized by \cite{Pensia} within the proof strategy previously developed in \cite{malachProvingLotteryTicket2020}. 

\subsection{Our Contribution}
\label{sec:contrib}

We answer the aforementioned question by introducing and proving a refined variant of the RSS Problem, namely the Random Fixed-Size Subset Sum Problem (\RFSS{}), in which the approximation of the target values should be achieved by only considering subsets of fixed size $k$ from a set of $n$ samples (Theorem~\ref{thm:srss}). 
We focus on subsets of fixed size $k$ rather than subsets of size up to $k$ for two main reasons. 
From a theoretical point of view, it is a stronger requirement, and practically speaking, using fixed-size subsets enables us to achieve SLTH results where the layers of the lottery ticket exhibit a uniform structure, potentially offering a computational advantage in their implementation.

% As shown in Section~\ref{sec:slth}, by substituting our \RFSS{}  result with linear subset size ($\gamma m = \Theta(m)$) to the RSS result in previous work, we essentially recover as special cases previous SLTH results such as \cite{Pensia,natale1,burkholzConvolutionalResidualNetworks2022,ferbachGeneralFrameworkProving2022}, among others.
%
% When the density vanishes ($\gamma\stackrel{m\rightarrow\infty} {\rightarrow}0$), we show how this impacts the \emph{overparameterization}, i.e., the ratio between the number of parameters of the original network $N_\Omega$ and that of the class of target networks $N_t$ that can be approximated by the subnetworks $N_S$ (cfr. Section \ref{sec:slth}). 
In Section~\ref{sec:slth}, we show how the density $\gamma$ impacts the \emph{overparameterization}, i.e., the ratio $(\nicefrac{m}{m_t})$ between the number of parameters of the original network $N_\Omega$ and that of the class of target networks $N_t$ that can be $\epsilon$-approximated by pruning $N_\Omega$ down to a subnetwork $N_S$ with $\gamma m$ parameters. 
In our analysis, we also compare and recover as special cases previous SLTH results such as \cite{Pensia,malachProvingLotteryTicket2020,dacunhaProvingStrongLottery2022,burkholzConvolutionalResidualNetworks2022,ferbachGeneralFrameworkProving2022}.
% For instance, when $\gamma m = \Theta(m_t)$, our Theorem \ref{thm:pensia} essentially matches the seminal result by \cite{malachProvingLotteryTicket2020}, which states that the overparameterization is $\Tilde{O}({\nicefrac{m_t^2}{\epsilon^2}})$, and when $\gamma m = \Theta(m)$ we recover \cite{Pensia} up to a logarithmic factor. 
For instance, when $\gamma m = \Theta(m)$, we recover up to a logarithmic factor the result of \cite{Pensia}, which states that the overparameterization needed is $O(\log_2{\left({\nicefrac{m_t^2}{\epsilon^2}}\right)})$. 
In the case of Dense Neural Networks, Theorem \ref{thm:pensia} thus bridges the gap between the two extreme cases of  $\gamma m = \Theta(m_t)$ and $\gamma m = \Theta(m)$ considered in \cite{malachProvingLotteryTicket2020} and \cite{Pensia}, respectively. It is worth noting that \cite{Pensia} is often considered an improvement over \cite{malachProvingLotteryTicket2020}, as it exponentially reduces the overparameterization, albeit at the cost of a trivial sparsity level.
Finally, we prove that our bounds on the overparameterization as a function of the subnetwork sparsity are essentially tight.

\paragraph{Organization of the paper.}

After reviewing the literature on the SLTH in Section~\ref{sec:related}, we introduce the Random Fixed-Size Subset Sum Problem in Section~\ref{sec:RFSS}. In Section~\ref{sec:slth}, we explore some applications of the \RFSS{} Problem to the SLTH, and finally draw our conclusions in Section~\ref{sec:conclusions}. Some limitations of our work, along with its potential impact, are discussed in Section~\ref{sec:lim_impact}.

\section{Related Work}
\label{sec:related}

% \subsubsection{Related Work on the SLTH}
% \label{sec:related_slth}

The SLTH is named after the LTH, which was introduced by Frankle and Carbin in \cite{frankleLotteryTicketHypothesis2018}. At the time of writing, this paper has received over 3,300 citations, attesting to the significance and impact of the research topic.
Surveying the LTH is thus besides the scope of this work, and we defer the reader to dedicated surveys such as \cite{liuSurveyLotteryTicket2024}. 

The SLTH was empirically motivated by work investigating training-by-pruning algorithms such as  \cite{zhouDeconstructingLotteryTickets2019,ramanujanWhatHiddenRandomly2020}, namely algorithms that leverage the gradient of the network parameters to learn a good \emph{mask} of the edges to be retained (i.e., a good subnetwork, called the winning ticket).
\cite{zhouDeconstructingLotteryTickets2019} achieves this by learning a probability associated to each edge, which is then used to sample the edges that should be included in the subnetwork. 
\cite{ramanujanWhatHiddenRandomly2020} gets rid of the stochasticity involved in the aforementioned strategies by learning a score associated to each edge; the subnetwork is then determined by including the edges with the highest score. 
Such strategies are leveraged in \cite{isikSparseRandomNetworks2022,isikAdaptiveCompressionFederated2024}  in a federated learning setting, in order to improve the communication cost of distributed training by communicating the sampled masks of a fixed shared network, rather than the entire weights. 
However, these training-by-pruning algorithms are generally not computationally less expensive than classical training, since they also make use of backpropagation to update scores and are applied to a sufficiently large network to find a winning ticket. To reduce the computational cost of finding a good subnetwork, \cite{gadhikarWhyRandomPruning2023} shows, both theoretically and experimentally, that randomly pre-pruning the source network before looking for a winning ticket can be an effective approach.  In \cite{otsukaPartialSearchFrozen2024a}, on top of randomly pruning the source network, some parameters are also frozen. Frozen parameters are forced to be part of the winning ticket and they do not have an associated score, which effectively reduces the search space for the training-by-pruning algorithms.

The first rigorous proof of the SLTH in the case of dense neural networks has been provided by \cite{malachProvingLotteryTicket2020}, which establishes a framework that was inherited by the subsequent works.
\cite{Pensia} crucially shows that the framework in \cite{malachProvingLotteryTicket2020} allows the application of the RSS analysis in \cite{Lueker98}, proving that, with no constraint on the size of the subnetworks, a random network with $m$-parameters can be pruned to approximate target networks with $m/\log(1/\epsilon)$ parameters (we defer the reader to Theorem \ref{thm:pensia} for details on further constraints on the parameters).
An alternative proof of the result in \cite{Pensia} was simultaneously shown in
\cite{orseauLogarithmicPruningAll2020}.
\cite{dacunhaProvingStrongLottery2022} and \cite{burkholzConvolutionalResidualNetworks2022} successively extended \cite{Pensia} and \cite{orseauLogarithmicPruningAll2020} to convolutional neural networks (CNNs). 
By leveraging multidimensional generalizations of RSS \cite{dacunhaRevisitingRandomSubset2023,borstIntegralityGapBinary2023}, \cite{cunhaPolynomiallyOverParameterizedConvolutional2023} further extended the SLTH to structured pruning of CNNs and, as a special case, dense networks. 
Finally, \cite{ferbachGeneralFrameworkProving2022} provided a general framework that proves the SLTH for equivariant networks. 

As for refinements and generalizations of the above results, 
\cite{burkholzMostActivationFunctions2022} shows that, at the cost of a quadratic overhead in the overparameterization w.r.t. \cite{Pensia}, the number of layers of the random network $N_\Omega$ can be reduced to $\ell+1$, where $\ell$ is the number of layers of the target networks $N_z$; furthermore, while previous results only considered networks with ReLU activation, \cite{burkholzMostActivationFunctions2022} shows how to extend the proof in \cite{Pensia} to a more general class of activations functions. 
\cite{burkholzExistenceUniversalLottery2022} introduces the notion of universal lottery ticket, and show that it is possible to prune a sufficiently overparameterized random network so that the resulting subnetwork (the lottery ticket) can approximate certain class of functions up to an affine transformation of the output of the subnetwork (in this sense being universal). 
\cite{fischerLotteryTicketsNonzero2022} shows how to extend the proof in \cite{Pensia} when neurons have random biases, and adapts the training-by-pruning algorithm of \cite{ramanujanWhatHiddenRandomly2020} to find a strong lottery ticket with a desired sparsity level. 
%
% \textcolor{red}{Ema: the next one also says that provides insights on the existence of "sparse" subnetworks but it is very unclear to me.. please help checking.}\fnote{What is the sigma they have in Thm2? Can't find the definition. It seems like if it's 1 just like in our case you get that $\lambda=1$ which is dubious.}
Motivated by theoretical insights on the existence of sparse strong lottery tickets, \cite{fischerPlantSeekCan2022} develops a framework to plant the latter in large random network and investigates training-by-pruning algorithms, providing evidence that sparse strong lottery tickets typically exists for common machine learning tasks, and the difficulty to find them is of algorithmic nature. \futurenote{Not clear what they do with the planted thingy. Are they saying that when you trian by pruning you get the planted network?}

Our proof of the \RFSS{} Problem in Section \ref{sec:RFSS} is based on the second moment method approach first explored by \cite{Lueker82}, and which has recently been refined to prove multidimensional generalizations of RSS by \cite{dacunhaRevisitingRandomSubset2023} and \cite{borstIntegralityGapBinary2023}.
%
% Finally, it is worth mentioning works such as \cite{borgsPhaseTransitionFinitesize2001}  and \cite{borgsPhaseDiagramConstrained2004} which have studied the
% Number Partition Problem, a closely related to the RSS Problem by a simple reduction.

\section{Fixed-Size Random Subset Sum}
\label{sec:RFSS}

In this section we present our technical contributions on the \RFSS{}, which are the foundation of our proofs regarding the sparsity of the SLTH. 

Let us start by introducing some notation. We denote by $[n]$ the set $\{1, \ldots, n\}$, for $n \in \mathbb{N}$. Given a set $\Omega=\left\{ X_{1},...,X_{n}\right\} $ and a set of
indices $S\subseteq\left[n\right]$ we define $\Sigma_{S}^{\Omega}=\sum_{i\in S}X_{i}$,
and we omit $\Omega$ when clear from the context.
We now define a class of distributions for which our \RFSS{} result holds.

\begin{definition}[\almostGaussian{}]
\label{def:quasi_unif}
    We say that a probability density function
    $f$ is \emph{\almostGaussian{}} if there exist positive constants $c_{l}$
    and $c_{u}$ such that, for all $k\in\mathbb{N}$,
    given $k$ independent samples $X_{1},...,X_{k}$ with density $f$,
    the density of their sum $f_{\Sigma_{\left[k\right]}}$ satisfies 
    \[
        \frac{c_{l}}{\sqrt{k}}\leq f_{\Sigma_{\left[k\right]}}\left(x\right)\leq\frac{c_{u}}{\sqrt{k}},
    \]
    with the lower bound holding for all $x\in\left[-\sqrt{k},\sqrt{k}\right]$ and the upper bound holding for all $x\in\mathbb{R}$.
\end{definition}
At first, our definition of \almostGaussian{} could look as a weaker version of a classical local limit theorem on the sum of random variables (e.g., see \cite[Chapter VII, Theorem 7]{petrov1975sums}). 
However, that is not the case, since we require a lower bound on the sum for any $k$, which is needed to prove our main result.

 Denote, for all $x\in[0,1]$, the binary entropy as \[H_2(x)=-x\log_2 x-(1-x)\log_2(1-x).\]
 Our main technical result is the following proof of a fixed-size subset variant of the RSS Problem.

\begin{theorem}
    \label{thm:srss}
    Let $0<\varepsilon<1$, $c_{\text{hyp}}\ge 1$, $k,n$ be integers
    with $1\leq k\leq \frac{n}{2}$, and let $\Omega=\left\{ X_{1},...,X_{n}\right\} $
    where the $X_{i}$'s are i.i.d. random variables with \almostGaussian{}
    density. There exists a constant $c_{\text{thm}}$
    such that, if 
    \begin{equation}
        \label{eq:bound_n}
        n\geq c_{\text{hyp}}\frac{\log_{2}\frac{k}{\varepsilon}}{H_{2}\left(\frac{k}{n}\right)},
    \end{equation}
    then for every fixed $z\in\left[-\sqrt{k},\sqrt{k}\right]$ 
    it holds that
    \[
        \Pr\left(\exists S\subset\left[n\right],\left|S\right|=k:\left|\Sigma_{S}-z\right|<\epsilon\right)\geq c_{\text{thm}}.
    \]
\end{theorem}
\begin{remark}
The proof of Theorem \ref{thm:srss} is given in Section \ref{sec:rfss_proof}, and it actually holds for any $1\leq k\leq \lambda n$, for an arbitrary $\lambda\in[\nicefrac{1}{n},1)$. 
We state the theorem this way for readability and because we are primarily interested in high-sparsity settings (i.e., small size $k$ of the subsets), so considering values of $k \ge \frac{n}{2}$ does not add much to our analysis.
The same remark also holds for Corollary \ref{cor:SRSS_amp}.
\end{remark}

% The feasibility of the condition in Eq. \ref{eq:bound_n} is straightforward.\footnote{It suffices to consider, for example, all $n\ge k(\nicefrac{k}{\epsilon})^{\nicefrac{c_{\text{hyp}}}{k}}$, and note that $nH_2(\nicefrac{k}{n})\ge k\log_2(\nicefrac{n}{k})$.}
%
The \almostGaussian{} condition of Definition \ref{def:quasi_unif} is easily verified for distributions such as the Gaussian distribution. 
Previous SLTH results rely on a classical resampling argument by \cite[Corollary 3.3]{Lueker98}, which shows how RSS results for Uniform$[-1,1]$ independent random variables naturally extend to independent random variables that \emph{contains} a uniform distribution, in the sense that they can be expressed as the mixture of distributions one of which is Uniform$[-1,1]$ with constant probability.\footnote{The definition in \cite[Corollary 3.3]{Lueker98} is actually more general, since it concerns a different problem.}
%For the sake of being self-contained, we still provide a proof in Appendix~\ref{apx:contain_unif}.
The next lemma thus proves that the Uniform$[-1,1]$ distribution is \almostGaussian{}\footnote{We believe that Lemma \ref{lem:uniform} is known, but we could not find a reference.}.
A detailed proof is provided in Appendix \ref{apx:uniform_has_property}.
\begin{lemma}
    \label{lem:uniform}
    The Uniform$[-1,1]$ probability density function is \almostGaussian{}, i.e., given a set $\mathcal U_n = \{U_i\}_{i\in [n]}$ of i.i.d. variables $U_i$ with Uniform$[-1,1]$ probability density function, there exist constants $\Cmin$ and $\Cmax$ such that the probability density function $f(x,n)$ of the sum $\Sigma^{\mathcal U_n}_{[n]}$ of these variables, for all $n\in\mathbb{N}$,
    \begin{equation}
        \label{eq:uniform_property}
        \frac{\Cmin}{\sqrt{\indexn}} \leq \f(x,{\indexn}) \leq \frac{\Cmax}{\sqrt{\indexn}},
    \end{equation}
    with the lower bound holding for all $ x \in [-\sqrt{\indexn}, \sqrt{\indexn}]$, and the upper bound holding for all $x\in\mathbb{R}$.
    %where $\f(x,n)$ is the probability density function of $\Sigma^{\mathcal U_n}_{[n]}$. 
\end{lemma}

Finally, in our proofs on the Sparse SLTH in Section \ref{sec:slth},  we make use of the following corollary of Theorem \ref{thm:srss}, which ensures a uniform high probability of hitting any target $z\in[-\sqrt{k},\sqrt{k}]$, considering independent random variables that contain a uniform distribution.
\begin{corollary}
    \label{cor:SRSS_amp}
    Let $0<p\leq 1$ and $\epsilon \in (0,\nicefrac{1}{2})$ be constants, $k,n$ with $1\leq k\leq \frac{n}{2}$, 
    and let $\Omega=\left\{ X_{1},...,X_{n}\right\} $
    be i.i.d. random variables whose density is a mixture of a Uniform$([-1,1])$ with probability $p$, and some other density otherwise.
    There exists a positive constant $c_{\text{amp}}$ that only depends on $p$ 
    such that, if 
    \begin{equation}
    \label{eq:n_cond_amp}
        n\geq c_{\text{amp}}\frac{\log_{2}^2\frac{k}{\varepsilon}}{H_{2}\left(\frac{k}{n}\right)}, 
    \end{equation}
    then 
    \[
        \Pr\left(\forall z\in\left[-\sqrt{k},\sqrt{k}\right], \exists S_z\subset\left[n\right],\left|S_z\right|=k:\left|\Sigma_{S_z}-z\right|<\epsilon\right)\geq 1-\epsilon.
    \]
\end{corollary}
\begin{proof}[Proof Idea.]
    The corollary follows from three arguments.
    First, by a standard sampling argument,
    %which we provide in Lemma~\ref{lem:unif_containing} in Appendix \ref{apx:contain_unif}, 
    we can assume that a constant fraction of the sample follows a Uniform$[-1,1]$ distribution.
    Secondly, by Lemma \ref{lem:uniform}, the uniform probability density function is \almostGaussian{}. We can thus apply Theorem~\ref{thm:srss}, which guarantees a success probability of $c_{\text{thm}}$ for approximating a given target.
    Finally, by a standard probability amplification argument and a union bound applied to Theorem \ref{thm:srss}, by paying an extra factor $\log_2(k/\epsilon)$ in Eq. \ref{eq:bound_n}, the constant $c_{\text{thm}}$ can be \futurenote{assumed?}assumed to be $1-\epsilon$, and the existence of a suitable subset $S_z$ holds simultaneously for all $z\in[-\sqrt{k},\sqrt{k}]$. Details are given in Appendix~\ref{apx:amplif}.
    % First, by Lemma \ref{lem:uniform}, the sum of uniform random variables is \almostGaussian{} and, by a standard sampling argument which we provide in Lemma~\ref{lem:unif_containing} in Appendix \ref{apx:contain_unif}, the results holds for any mixture of densities which includes a Uniform$[-a,a]$ density for some constant $a>0$ with constant probability.
    % Secondly, by a standard probability amplification argument and a union bound applied to Theorem \ref{thm:srss}, by paying an extra factor $\log_2(k/\epsilon)$ in Eq. \ref{eq:bound_n}, the constant $c_{\text{thm}}$ can be assumed to be $1-\epsilon$, and the existence of a suitable subset $S_z$ holds simultaneously for all $z\in\left[-\sqrt{k},\sqrt{k}\right]$ (a detailed proof is given in Appendix \ref{apx:amplif}).
    %
\end{proof}

For $k$ big enough, we can get rid of the squared logarithmic dependency on $k$ in the right hand side of Equation \ref{eq:n_cond_amp}, as shown in the following Corollary, whose proof can be found in Appendix \ref{apx:amplif_simp}.

\begin{corollary}
    \label{cor:SRSS_amp_simp}
        Let $0<p\leq 1$ and $\epsilon \in (0,\nicefrac{1}{2})$ be constants, $k,n$ be integers with $1\leq k\leq \frac{n}{2}$
        and $k \geq 2 c_{\text{amp}} \left(\log_{2}^2 k + 2log_2{k} \cdot  log_2{\frac{1}{\varepsilon}}\right)$.
        Let $\Omega=\left\{ X_{1},...,X_{n}\right\} $
        be i.i.d. random variables whose density is a mixture of a Uniform$([-1,1])$ with probability $p$, and some other density otherwise.
        There exists a positive constant $c_{\text{amp}}$ that only depends on $p$ 
        such that, if 
        \begin{equation}
            \label{eq:n_cond_amp_simp}
            n\geq 2 c_{\text{amp}}\frac{\log_{2}^2\frac{1}{\varepsilon}}{H_{2}\left(\frac{k}{n}\right)}, 
        \end{equation}
        then 
        \[
            \Pr\left(\forall z\in\left[-\sqrt{k},\sqrt{k}\right], \exists S_z\subset\left[n\right],\left|S_z\right|=k:\left|\Sigma_{S_z}-z\right|<\epsilon\right)\geq 1-\epsilon.
        \]
    \end{corollary}

As customary in conference versions of papers, our proofs adopt the convention of taking ceilings and floors as suitable for non integer fractional terms. This is done in the interest of the reader (and ours), and does not impact the results in any significant way.
%We strongly believe that this does not impact our results in any significant way.

\subsection{Proof of Theorem  \ref{thm:srss}}
\label{sec:rfss_proof}

\begin{proof}[Proof of Theorem  \ref{thm:srss}]
For simplicity, throughout the proof we will often use $c$
to denote any positive constant.
%\footnote{Given two functions $f\left(n\right)$ and $g\left(n\right)$, $f\left(n\right)\leq cg\left(n\right)$is thus equivalent to $f\left(n\right)=O\left(g\left(n\right)\right).$}
Let $\mathcal{S}_{k}=\{S\subset[n]\,|\,|S|=k\}$ and define, for a fixed $z\in[-\sqrt{k},\sqrt{k}]$,
\[
Y=Y(z)=\sum_{S\in\mathcal{S}_{k}}Z_{S}
\]
where $Z_{S}=Z_S(z)=\mathbf{1}_{\{\left|\Sigma_{S}-z\right|<\epsilon\}}$. 
Following \cite{Lueker82}, we exploit the second moment method for \RFSS, generalising it to arbitrary $k$.
\begin{equation}
    \label{eq:2ndmethod}
    \Pr\left(Y>0\right) \geq \frac{\left(\mathbb{E}\left[Y\right]\right)^{2}}{\mathbb{E}\left[Y^{2}\right]},
\end{equation}
it thus suffices to prove that 
\begin{equation}
\mathbb{E}\left[Y^{2}\right]\leq c\left(\mathbb{E}\left[Y\right]\right)^{2}.\label{eq:goal_with_expec}
\end{equation}

We first rewrite Eq. \ref{eq:2ndmethod} in a more convenient form.
Let $\tilde{S}$ and $\tilde{S}^{\prime}$ be two independently and
uniformly at random chosen subsets of $\left[n\right]$ of size $k$,
and denote $H_S(z)$ as the event that $\Sigma_S$ $\varepsilon$-approximates $z$, namely 
\[ 
    \Hz{S}=H_S(z)=\left\lbrace\left|\Sigma_{S}-z\right|<\epsilon\right\rbrace .
\]

We have 
\begin{equation}
\mathbb{E}[Y]=\sum_{S\in\mathcal{S}_{k}}\mathbb{E}[Z_{S}]=\sum_{S\in\mathcal{S}_{k}}\Pr(\Hz{S})=\binom{n}{k}\Pr\left(\Hz{\tilde{S}}\right)\label{eq:rewrite_expec}
\end{equation}
and 
\begin{align}
\mathbb{E}[Y^{2}] & =\mathbb{E}\left[\left(\sum_{S\in\mathcal{S}_{k}}Z_{S}\right)\left(\sum_{S'\in\mathcal{S}_{k}}Z_{S'}\right)\right]=\sum_{S,S'\in\mathcal{S}_{k}}\mathbb{E}\left[Z_{S}Z_{S'}\right]\nonumber \\
 & =\sum_{S,S'\in\mathcal{S}_{k}}\Pr\left(\Hz{S}\wedge \Hz{S^{\prime}}\right)=\binom{n}{k}^{2}\Pr\left(\Hz{\tilde{S}}\wedge \Hz{\tilde{S}^{\prime}}\right).\label{eq:rewrite_second_moment}
\end{align}

Using Eqs. \ref{eq:rewrite_expec} and \ref{eq:rewrite_second_moment}
we can rewrite the r.h.s. of Eq. \ref{eq:2ndmethod} as follows
\[
\frac{\left(\mathbb{E}\left[Y\right]\right)^{2}}{\mathbb{E}\left[Y^{2}\right]}=\frac{\left[\Pr\left(\Hz{\tilde{S}}\right)\right]^{2}}{\Pr\left(\Hz{\tilde{S}}\wedge \Hz{\tilde{S}^{\prime}}\right)}=\frac{\Pr\left(\Hz{\tilde{S}}\right)}{\Pr\left(\Hz{\tilde{S}^{\prime}}\,|\,\Hz{\tilde{S}}\right)}.
\]
Eq. \ref{eq:goal_with_expec} thus becomes 
\begin{equation}
\Pr\left(\Hz{\tilde{S}^{\prime}}\,|\,\Hz{\tilde{S}}\right)\leq c\Pr\left(\Hz{\tilde{S}}\right).\label{eq:goal_with_probs}
\end{equation}
 Let $I_{i}$ denote the event $\{|\tilde{S}\cap\tilde{S}^{\prime}|=i\}$
and $I_{a,b}$ the event $\bigcup_{a\leq i\leq b}I_{i}$. Fix $\mu\in(\lambda,1)$. By the law
of total probability and independence of $I_i$ and $\Hz{\tilde{S}}$, we rewrite the l.h.s. of Eq. \ref{eq:goal_with_probs}
as follows: 
\begin{align}
 & \Pr\left(\Hz{\tilde{S}^{\prime}}\,|\,\Hz{\tilde{S}}\right)\nonumber \\
 & = \Pr\left(\Hz{\tilde{S}^{\prime}}\wedge I_{k}\,|\,\Hz{\tilde{S}}\right)+\Pr\left(\Hz{\tilde{S}^{\prime}}\wedge I_{\mu k,k-1}\,|\,\Hz{\tilde{S}}\right)+\Pr\left(\Hz{\tilde{S}^{\prime}}\wedge I_{0,\mu k-1}\,|\,\Hz{\tilde{S}}\right)\nonumber \\
 & =\Pr\left(I_{k}\right)\cdot\Pr\left(\Hz{\tilde{S}^{\prime}}\,|\,\Hz{\tilde{S}},I_{k}\right)\label{eq:first_term}\\
 & \qquad+\Pr\left(I_{\mu k,k-1}\right)\cdot\Pr\left(\Hz{\tilde{S}^{\prime}}\,|\,\Hz{\tilde{S}},I_{\mu k,k-1}\right)\label{eq:second_term}\\
 & \qquad+\sum_{i=0}^{\mu k-1}\left(\Pr\left(I_{i}\right)\cdot\Pr\left(\Hz{\tilde{S}^{\prime}}\,|\,\Hz{\tilde{S}},I_{i}\right)\right).\label{eq:third_term}
\end{align}
To conclude the proof, it suffices to show that each addendum in Eqs.
\ref{eq:first_term}, \ref{eq:second_term} and \ref{eq:third_term}
are upper-bounded by some constant multiple of $\nicefrac{\epsilon}{\sqrt{k}}$, since the lower bound in Definition \ref{def:quasi_unif} ensures that
\begin{equation}\label{eq:make_prob_appear}
\frac{\epsilon}{\sqrt{k}}\le c\Pr\left(\Hz{\tilde{S}}\right).
\end{equation}

As for the first addendum (Eq. \ref{eq:first_term}), since $\Pr\left(\Hz{\tilde{S}^{\prime}}\,|\,\Hz{\tilde{S}},I_{k}\right)=1$,
then 
\begin{align}
& \Pr\left(I_{k}\right) \cdot \Pr\left(\Hz{\tilde{S}^{\prime}}\,|\,\Hz{\tilde{S}},I_{k}\right)  =\Pr\left(I_{k}\right)=\frac{1}{{n \choose k}}
 \stackrel{(a)}{\leq} \sqrt{\frac{8k(n-k)}{n}}2^{-nH_2\left(\frac{k}{n}\right)}\nonumber \\
 & 
 \stackrel{(b)}{\leq} \sqrt{\frac{8k(n-k)}{n}}2^{-c_{\text{hyp}}\log_2\frac{k}{\epsilon}}
 \stackrel{(c)}{\leq}2\sqrt{2}\frac{\epsilon}{\sqrt{k}},\label{eq:square_root_first_addendum}
\end{align}
where inequality $(a)$ in Eq. \ref{eq:square_root_first_addendum}
is a standard lower bound on ${n \choose k}$ holding for all $k\leq n-1$; 
in inequality $(b)$ in Eq. \ref{eq:square_root_first_addendum} we used Eq. \ref{eq:bound_n}, namely $nH_{2}\left(\frac{k}{n}\right)\geq c_{\text{hyp}}\log_{2}\frac{k}{\varepsilon}$;
in inequality $(c)$ in Eq. \ref{eq:square_root_first_addendum}
we used that $c_{\text{hyp}}\ge 1$. 

As for the second addendum (Eq. \ref{eq:second_term}), we next show
that 
\begin{equation}\label{eq:square_root_second_addendum}
\Pr\left(I_{\mu k,k-1}\right)\Pr\left(\Hz{\tilde{S}^{\prime}}\,|\,\Hz{\tilde{S}},I_{\mu k,k-1}\right)\leq c\frac{\epsilon}{\sqrt{k}}
\end{equation}
by proving that
\begin{equation}
\Pr\left(I_{\mu k,k-1}\right)\leq\frac{c}{\sqrt{k}}\label{eq:first_factor_second_term}
\end{equation}
 and 
\begin{equation}
\Pr\left(\Hz{\tilde{S}^{\prime}}\,|\,\Hz{\tilde{S}},I_{\mu k,k-1}\right)\leq c\varepsilon.\label{eq:second_factor_second_term}
\end{equation}
First, observe that  $I=|\tilde{S}\cap\tilde{S}^{\prime}|$ follows
a Hypergeometric$(n,k,k)$ distribution, thus by Chebyshev's inequality
%\global\long\def\var{}
\begin{align}
\Pr\left(I_{\mu k,k-1}\right) &\leq\Pr\left(I\geq\mu k\right)=\Pr\left(I-\frac{k^2}{n}\geq\mu k-\frac{k^2}{n}\right)
 \leq\frac{\Var{I}}{\mu^2k^{2}\left(1-\frac{k}{\mu n}\right)^2}
 \notag\\&\le c'\frac{\frac{k^{2}}{n}\frac{n-k}{n}\frac{n-k}{n-1}}{k^{2}}
 \leq\frac{c}{\sqrt{k}},\nonumber
\end{align}
having set $c'=\mu^2(1-\nicefrac{\lambda}{\mu})^2>0$, thus proving Eq. \ref{eq:first_factor_second_term}. 
The proof of Eq.~\ref{eq:second_factor_second_term} is given in Appendix \ref{apx:srss_proof_details}, concluding the proof of Eq. \ref{eq:square_root_second_addendum}.

As for the third addendum (Eq. \ref{eq:third_term}), in Appendix \ref{apx:srss_proof_details} we show that
\begin{align}
 \sum_{i=0}^{\mu k-1}\Pr\left(I_{i}\right)\cdot\Pr\left(\Hz{\tilde{S}^{\prime}}\,|\,\Hz{\tilde{S}},I_{i}\right)
 \le c\frac{\epsilon}{\sqrt{k}},\label{eq:square_root_third_addendum}
\end{align}

The three bounds on the addenda in Eqs. \ref{eq:first_term},
\ref{eq:second_term}, and \ref{eq:third_term} (Eqs.
\ref{eq:square_root_first_addendum}, \ref{eq:square_root_second_addendum},
and \ref{eq:square_root_third_addendum}, respectively), combined with Eq. \ref{eq:make_prob_appear}, conclude the proof. 
\end{proof}

\section{Sparse Strong Lottery Ticket Hypothesis (SSLTH)}
\label{sec:slth}

We now apply our results on the \RFSS{} problem to the SLTH and obtain guarantees on the sparsity of winning tickets for Dense Neural Networks (DNNs, Theorem~\ref{thm:pensia}) and Equivariant NNs (Theorem~\ref{thm:ferbach}).

The next theorem essentially interpolates between the two extremes of \cite{malachProvingLotteryTicket2020}[Theorem 2.1] (where $\gamma m = \Theta(m_t)$) and \cite{Pensia}[Theorem 1] (where $\gamma m = \Theta(m)$), where we recall that $m$ and $m_t$ represent the number of parameters of the overparameterized and the target networks, respectively, and $\gamma$ is the density of the winning ticket.

% The next theorem essentially generalizes \cite{Pensia}[Theorem 1] up to a logarithmic factor  (when $\gamma= \nicefrac{1}{2}$), and \cite{malachProvingLotteryTicket2020} up to a factor $\epsilon^{-\Theta(1)}$ (when $k=1$).
We use $\sigma(\cdot)$ to denote the ReLU activation function, i.e., $\sigma(x) = x\cdot {\bf 1}_{x\ge 0}$, and $\|\mathbf{W}\|$ to denote the spectral norm of the matrix $\mathbf{W}$. 
Let $\mathcal F$ to be a set of target ReLU neural networks $f : \mathbf{R}^{d_0}\to \mathbf{R}^{d_l}$ of depth $l$ such that
    \begin{align}
        \mathcal{F} = \{f: f(\mathbf{x})=\mathbf{W}_l \sigma(\mathbf{W}_{l-1}\dots\sigma(\mathbf{W}_1 \mathbf{x})),\,
        \forall i \,\,\,
        \mathbf{W}_i \in \mathbb{R}^{d_i \times d_{i-1}} 
        \text{ and }   \|\mathbf{W}_i\| \leq 1 \}
    \label{EqnFFamilyUp}
    \end{align}
For a given $f\in\mathcal{F}$, for all $i\in[\ell]$, let $\rho_i=\max\{\nicefrac{d_{i-1}}{d_i},\nicefrac{d_i}{d_{i-1}}\}$, and $\rho=\max_i\rho_i$. Then, recalling that $c_{\text{amp}}$ is the constant defined in Corollary~\ref{cor:SRSS_amp}, we have the following result. 
\begin{theorem}[SSLTH for DNNs]
    \label{thm:pensia}
    Let $g$ be a randomly initialized feed-forward $2\ell$-layer neural network, in which each weight is drawn from a Uniform$[-1,1]$ distribution, of the following form:
    \[g({\bf x})=\mathbf{M}_{2l} \sigma(\mathbf{M}_{2l-1}\dots\sigma( \mathbf{M}_1 \mathbf{x})).\]    
    Let $\gammastar=\gammastar(%d_0,...,d_\ell,\ell, 
    \epsilon) \in (0,1)$,   
    $\mathbf{M}_{2i} \in \mathbb{R}^{d_i \times 2 d_{i-1}n_i^*}$ and $\mathbf{M}_{2i-1} \in \mathbb{R}^{2 d_{i-1}n_i^* \times d_{i-1}}$, with $n_i^*$ satisfying
    \begin{equation}\label{n^*}
        n_i^*= c_{\text{amp}}\frac{\log_{2}^2\left(\frac{ 2\ell d_{i-1}d_{i} \gammastar n^*_i}{\varepsilon}\right)}{H_{2}(\gammastar)}.
    \end{equation}
    With probability at least $1-\epsilon$,
    for every $f \in \mathcal F$, where $\mathcal F$ is defined as in Eq. \ref{EqnFFamilyUp},
    $g$ can be pruned to obtain a subnetwork of sparsity at least $\alpha = 1 - \gamma$ that approximates $f$ up to an error $\epsilon$, having defined $\gamma=\rho \gammastar$.

\end{theorem}
% \begin{remark}
% \dnote{Work in progress, delicate...}
% The pruning of $g$ is done following the same approach as \cite{Pensia} but using our Corollary \ref{cor:SRSS_amp} to obtain \fixedsize subsets (see Appendix \ref{apx:new_pensia} for details).
% Note that the total number of parameters of $g$ is $m_g=\sum_{i=0}^l 2 d_{i-1} d_{i}n_i^*$. After applying the same approach as \cite{Pensia} to prune $g$, but using our Corollary \ref{cor:SRSS_amp} to obtain \fixedsize subsets,
% we are left with a subnetwork $z$ with a total number of parameters of $m_z \approx \sum_{i=0}^l 2 d_{i-1} d_{i} \gamma n_i^* = \gamma \sum_{i=0}^l 2 d_{i-1} d_{i} n_i^* = \gamma m_g$, since $\gamma$ does not depend on the layer $i$. This is why in the end we get a subnetwork of approximately $(1-\gamma)$ sparsity.
% % Note that the total number of parameters of $g$ is $m_g=\sum_{i=0}^l 2 d_{i-1} d_{i}n_i^*$. After applying the same approach as \cite{Pensia} to prune $g$, but using our Corollary \ref{cor:SRSS_amp} to obtain \fixedsize subsets,
% % we are left with a subnetwork $z$ with a total number of parameters of $m_z \approx \sum_{i=0}^l 2 d_{i-1} d_{i} \gamma n_i^* = \gamma \sum_{i=0}^l 2 d_{i-1} d_{i} n_i^* = \gamma m_g$, since $\gamma$ does not depend on the layer $i$. This is why in the end we get a subnetwork of approximately $(1-\gamma)$ sparsity.
% \end{remark}
%
\begin{proof}[Proof Idea.]
    The theorem follows from a slight variation of the same approach detailed in \cite{Pensia}, in which we use our Corollary \ref{cor:SRSS_amp}  
    instead of \cite{Lueker98}[Corollary 2.5] when pruning $g$, allowing us to have control over the size of the pruned network. 
    A detailed proof is provided in Appendix \ref{apx:new_pensia}.
\end{proof}

To illustrate a simple example of how Theorem \ref{thm:pensia} addresses the main question asked in the introduction, consider the case where we want to approximate a target network with $m_t$ parameters and $\ell$ layers, each of width $d$ (so $\rho=1$ and $\gamma'=\gamma$), by pruning an overparameterized network to achieve a desired sparsity level of $\alpha=1-\gamma$. 
The condition expressed by Equation \ref{n^*} in \Cref{thm:pensia} comes from the use of \Cref{cor:SRSS_amp} when pruning network $g$, as shown in the proof. If, instead of \Cref{cor:SRSS_amp}, we use its simplified variant \Cref{cor:SRSS_amp_simp}, it is easy to observe that Equation \ref{n^*} would become
\begin{equation}\label{n^*_simp}
        n_i^*= c_{\text{amp}}\frac{\log_{2}^2\left(\frac{ 2\ell d_{i-1}d_{i}}{\varepsilon}\right)}{H_{2}(\gammastar)}.
\end{equation}
Using this condition, Theorem \ref{thm:pensia} then tells us that we need to prune a randomly initialized network with twice as many layers and a number of parameters of the order of $d^2 \frac{log^2 \frac{\ell d^2}{\epsilon}}{H(\gamma)}$. 

We will now clarify the connection between Theorem \ref{thm:pensia} and the earlier results from \cite{malachProvingLotteryTicket2020} and \cite{Pensia}. Figure \ref{fig:plot-pensia-malach} provides a quick visual comparison.

\begin{figure}
    \centering
    \includegraphics[width=0.6\linewidth]{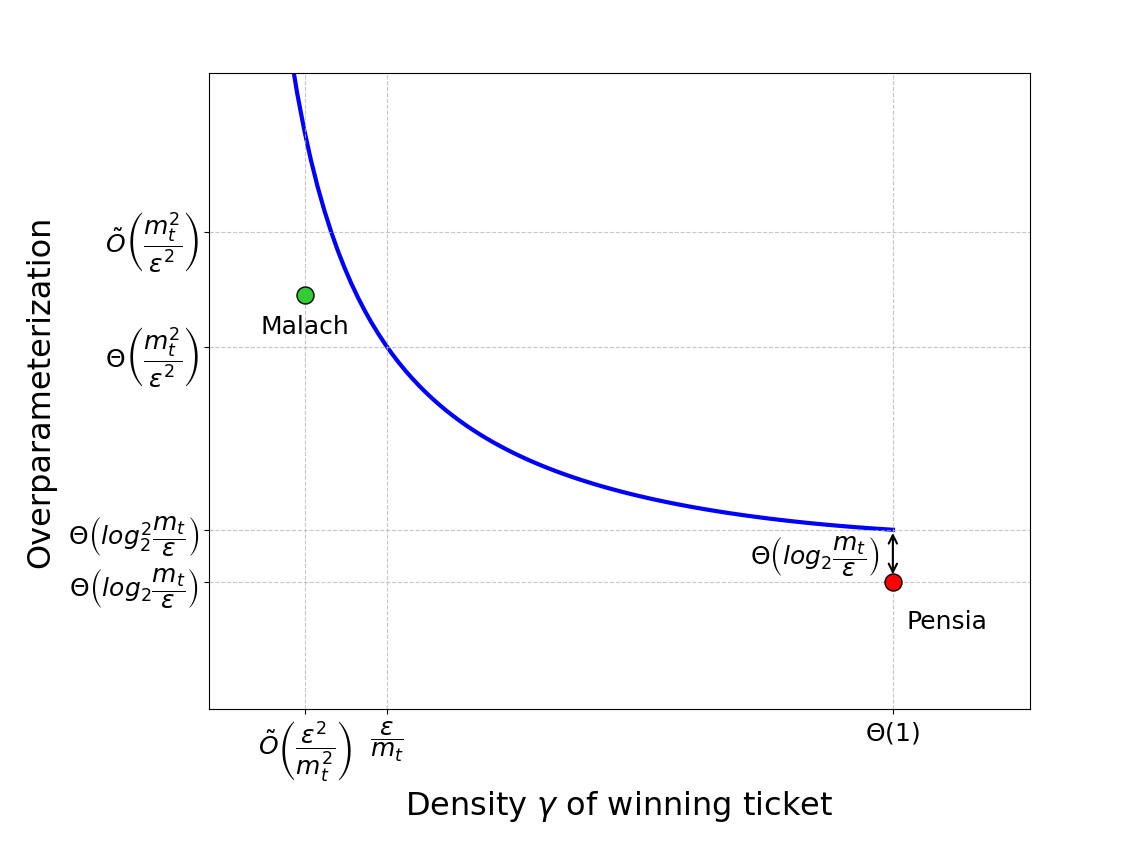}
    \caption{A qualitative plot showing the relationship between the density $\gamma$ of a winning ticket and the overparameterization required by Theorem \ref{thm:pensia} for a target network with $m_t$ parameters. Earlier results from Pensia et al. \cite{Pensia} and Malach et al. \cite{malachProvingLotteryTicket2020} are shown for comparison.}
    \label{fig:plot-pensia-malach}
\end{figure}

\paragraph{\normalfont{\textit{Malach et al.\cite{malachProvingLotteryTicket2020}.}}}
When all layers have the same width $d$, \cite{malachProvingLotteryTicket2020} showed that any target network with $l$ layers and a total of $m_t = d^2 l$ parameters can be $\epsilon$-approximated by pruning a randomly initialized network with $2l$ layers. The overparameterization of this network, relative to the target network, is $O\left(\frac{m_t^2}{\epsilon^2} \log_2 \frac{m_t}{\epsilon}\right) = \Tilde{O}\left(\frac{m_t^2}{\epsilon^2}\right)$. More specifically, the winning ticket found after pruning has a parameter count of the same order as the target network, resulting in a density of $\gamma = \Tilde{O}\left(\frac{\epsilon^2}{m_t^2}\right)$. Notably, this density $\gamma$ is the inverse of the overparameterization, as the size of the winning ticket matches that of the target network.

Next, we show that Theorem \ref{thm:pensia} also yields a density that is polynomial in $\frac{\epsilon}{m_t}$, when using an overparametrization of $\Theta\left(\frac{m_t^2}{\epsilon^2}\right)$. 
Let $z=\left(\frac{m_t}{\epsilon}\right)$, and note that $\gamma' = \gamma$ in Theorem \ref{thm:pensia}, since all layers have the same width.
As $n^*_i$ in Theorem \ref{thm:pensia} represents the overparametrization with respect to the target network, let us set $n^*_i = cz^2$, for some constant $c$.
Equation \ref{n^*} then becomes 
\begin{equation}
\label{eq:n^*_malach}
cz^2 \ge c_{\text{amp}} \frac{\log_2^2 (c z^3 \gamma)}{H(\gamma)} 
\end{equation}
We show that the inequality $cz^2 \ge c_{\text{amp}} \frac{\log_2^2 (c z^3 \gamma)}{\gamma \log_2(\nicefrac{1}{\gamma})}$ holds for some big enough constant $c$ when setting $\gamma = \frac{\epsilon}{m_t} = \frac{1}{z}$, which implies that Equation \ref{eq:n^*_malach} is also satisfied. 
We get $cz \ge c_{\text{amp}} \frac{\log_2^2 (c z^2)}{\log_2(z)}$, which is satisfied for a big enough constant $c$ (see Appendix \ref{apx:malach}).
Overall, when using an overparametrization $\Theta\left(\frac{m_t^2}{\epsilon^2}\right)$, we find a winning ticket with density $\frac{\epsilon}{m_t}$, as shown in Figure \ref{fig:plot-pensia-malach}.

\paragraph{\normalfont{\textit{Pensia et al.\cite{Pensia}.}}}
For simplicity, let us still consider target networks where all layers have the same width $d$, and we apply Theorem \ref{thm:pensia} using the simplified condition from Equation \ref{n^*_simp}.
When $\gamma m = \Theta(m)$, i.e. the density $\gamma$ is a constant as in \cite{Pensia} (see Appendix \ref{apx:pensia_lower}), the entropy term $H_{2}(\gammastar)$ in the right-side of Equation \ref{n^*_simp} also becomes a constant. In this setting, we indeed recover the result shown in \cite{Pensia}[Theorem 1], up to a logarithmic factor, as shown in Figure \ref{fig:plot-pensia-malach}.

Quite similarly to Theorem \ref{thm:pensia}, the next result essentially generalizes \cite{ferbachGeneralFrameworkProving2022} up to a factor $\log_2 \frac 1\epsilon$. The theorem is stated with the understanding that for $G$-equivariant networks, in order to preserve $G$-equivariance, pruning is best done not with respect to the parameters expressing the network in the canonical basis (i.e. directly on the weights of the network), but with respect to the \textit{equivariant parameters}, that is those coefficients expressing the linear layers of the network as a linear combination of the elements of the corresponding equivariant basis \cite{ferbachGeneralFrameworkProving2022}. For simplicity, due to the technical set-up, we assume all feature spaces being $\mathbb{F}=(\mathbb{R}^d,\sigma)$, with $\sigma$ the linear representation of the group $G$, and the same number $n$ of such feature spaces being stacked in each layer. A $G$-equivariant linear map from the $i$th feature space to the $i+1$st can be decomposed in a corresponding equivariant basis denoted $\mathcal{B}_{i\rightarrow i+1}=\mathcal{B}$. Since all feature spaces are the same, we omit the layers' indices. When stacking $n$ feature spaces in the input and output of the $i$th layer, the full equivariant basis is denoted $k_{n\rightarrow n}$, and finally the basis of the $G$-equivariant maps from $\mathbb{F}^n$ to $\mathbb{F}^n$ can be written as the Kronecker product $k_{n\rightarrow n}\otimes\mathcal{B}$. For any basis $\mathcal{B}=\{b_1,\ldots,b_p\}$, we denote its cardinality $p=|\mathcal{B}|$ and define $\Vert\mathcal{B}\Vert=\max_{\Vert\beta\Vert_\infty}\Vert\sum_{k=1}^p\beta_kb_k\Vert$, with $\Vert\cdot\Vert$ in the r.h.s. being the operator norm inherited from the $\ell_p$ norm.
\begin{theorem}[SSLTH for Equivariant Networks] 
\label{thm:ferbach} 
    Let $h$ be a random $2\ell$-layer $G$-equivariant network where all equivariant parameters are drawn from a Uniform$[-1,1]$ distribution, every odd layer 
    %has width $n$?
    expressed in the associated equivariant basis $k_{\tilde{n}\rightarrow n}\otimes \mathcal{B}$ and every even layer 
    %has width $n\tilde{n}$? 
    expressed in the associated equivariant basis $k_{n\rightarrow \tilde{n}}\otimes \mathcal{B}$. Let $\gamma=\gamma(\epsilon)\in(0,1)$, with $\tilde{n}$ satisfying
    \[
        \tilde{n}= c_{\text{amp}}\frac{\log_{2}^2\left(\frac{2\ell n^2\max\{|\mathcal{B}|,\Vert\mathcal{B}\Vert\}\gamma \tilde{n}}{\varepsilon}\right)}{H_{2}\left(\gamma\right)}.
    \]
    With probability at least $1-\epsilon$,
    for every $\ell$-layer $G$-equivariant neural network $f$, with all layers 
    %of width at most $n$? 
    expressed in the associated equivariant basis $k_{n\rightarrow n}\otimes \mathcal{B}$,
    $h$ can be pruned to obtain a $G$-equivariant subnetwork of sparsity at least $\alpha=1-\gamma$ that approximates $f$ up to an error $\epsilon$.
\end{theorem}
The proof, which we omit, is analogous to that of Theorem \ref{thm:pensia}, since \cite{ferbachGeneralFrameworkProving2022}[Theorem 1] exploits the exact same pruning strategy of \cite{Pensia}, except for the fact that it is applied not to the original parameters of the equivariant network, but to the network expressed in terms of its equivariant basis (the sparsity $\alpha$ is here also intended with respect to the equivariant parameters count). This allows the construction to apply without losing the property of equivariance in the pruned approximating subnetwork obtained. The crucial step is when Corollary \ref{cor:SRSS_amp} is applied in \cite{ferbachGeneralFrameworkProving2022}[Lemma 1], instead of \cite{Lueker98}[Corollary 2.5]. This is done in parallel, multiple times, across non-overlapping coefficients of the equivariant basis. Thanks to the careful preprocessing devised by the authors, this preserves equivariance and at the same time ensures that each application of Corollary \ref{cor:SRSS_amp} is independent of the others.

To conclude the section, we mention that Theorem \ref{thm:ferbach} applies in particular to vanilla CNNs, which are a special case of equivariant neural networks where the group is $G=(\mathbb{Z}^2,+)$, recovering previous SLTH results on CNN \cite{dacunhaProvingStrongLottery2022,burkholzConvolutionalResidualNetworks2022}. 
Furthermore, we remark that Theorem \ref{thm:pensia} can be revisited through the improvement upon the $2\ell$-depth overparameterization devised in \cite{burkholzMostActivationFunctions2022}, i.e., it is possible to provide sparsity guarantees also for overparameterizations requiring depth $\ell+1$ only. The analysis is more technical and we omit it, but the ideas are analogous to what shown in \cite{burkholzMostActivationFunctions2022}.
An analogous improvement is suggested as future work in \cite{ferbachGeneralFrameworkProving2022}.

\subsection{Lower bound on the required overparameterization}
\label{sec:LB}
We now adapt the lower bound of \cite{Pensia} in order to almost match the required overparameterization of our Theorem \ref{thm:pensia}, considering the simple scenario in which we want to approximate the family $\mathcal{F}$ of all linear networks with weights forming a matrix having spectral norm less than $\sqrt{k}$; more formally 
\begin{equation}
 \mathcal{F} := \{h_W : W \in \mathbb R^{d\times d}, \|W\| \leq \sqrt k\}, \quad \text  {   where   } \quad h_W(x) = W x.
\label{eq:LinearLower}
\end{equation}
The formal claim states that, if a network with $n$ parameters can approximate every $h_W \in \mathcal F$ with probability at least $\nicefrac{1}{2}$ (after it is pruned down to $k$ parameters), then the hypothesis of Theorem \ref{thm:srss} in Eq. \ref{eq:bound_n} must hold.\footnote{Equivalently, the hypothesis of Corollary \ref{cor:SRSS_amp} must hold up to a factor $\Theta(\log_{2}\frac{k}{\varepsilon})$.}
\begin{theorem}
    \label{thm:lowerBoundStrong}
    Let $n,k\in \mathbb N$, with $1\le k\leq \lambda n$, having set $\lambda=1-\nicefrac{1}{2\pi}\approx 0.84$. 
    Consider a neural network $g$ with $n$ parameters, and let $\mathcal{G}_k$ be the set of neural networks that can be formed by pruning $g$ down to $k$ parameters.
    Let $\mathcal{F}$ be as defined in Eq.~\ref{eq:LinearLower}. 
    If it holds that, for some $\epsilon<\nicefrac{1}{16}$,
    \begin{align}
    \forall {h_W\in \mathcal{F}} , \mathbb{P}\left( \exists g'\in \mathcal{G}_k: \max_{\mathbf{x}:\|x\|\leq 1} \|h_W(x)-g'(x)\| <\epsilon\right) \geq \frac{1}{2} ,
    \label{eq:lowerBdApprox}
    \end{align}
    then it holds that
    \[
        n\geq  \frac{d^2}{2} \frac{\log_{2}\frac{k}{\varepsilon}}{H_{2}\left(\frac{k}{n}\right)}.
    \]
\end{theorem}

The theorem follows by adapting the packing argument of \cite{Pensia}.  
A detailed proof is provided in Appendix \ref{apx:lower}.

\section{Conclusions}
\label{sec:conclusions}

In this work, we have extended previous results on the Strong Lottery Ticket Hypothesis by quantifying the required overparameterization as a function of the sparsity of the subnetworks. 
Central to our results is a proof of the Random Fixed-size Subset Sum (RFSS) Problem, a refinement of the seminal Random Subset Sum (RSS) Problem in which the subsets have a required fixed size. 

A challenging open problem is to extend our analysis of RFSS to the multidimensional case, in which the random samples and targets are vectors in $\mathbb R^d$. 
Previous extension of RSS to the Multidimensional RSS have indeed allowed to prove structured-pruning version of the SLTH \cite{dacunhaProvingStrongLottery2022}. 
A Multidimensional RFSS result would then allow to quantify, in the structured pruning case, the dependency of the overparameterization w.r.t. the sparsity of the (structured) subnetworks. 

Another future direction is to refine our analysis of the RFSS in Theorem \ref{thm:srss} in order to improve the probability of success to $1-\epsilon$ rather than constant, thus allowing to avoid shaving off the extra factor $\log_2(1/\epsilon)$ in our corollaries w.r.t. our lower bound, which is due to the amplification done in Corollary \ref{cor:SRSS_amp} to get to probability $1-\epsilon$.

Finally, an important future direction is to improve  training-by-pruning methods such as \cite{zhouDeconstructingLotteryTickets2019,ramanujanWhatHiddenRandomly2020,fischerLotteryTicketsNonzero2022,fischerPlantSeekCan2022,otsukaPartialSearchFrozen2024a} or to develop new ones, in order to allow to efficiently find strong lottery tickets of a desired sparsity, thus empirically validating our theoretical predictions.  

\section{Limitations and Impact}
\label{sec:lim_impact}
\paragraph{Limitations} Similar to all the research conducted on the LTH and the SLTH, this work only proves the existence of lottery tickets. To this date, it is not clear if these subnetworks can be found reliably (no formal proof exists) in an efficient manner - however, empirical evidence suggests that efficient algorithms exist (e.g., \cite{zhouDeconstructingLotteryTickets2019,ramanujanWhatHiddenRandomly2020}).

\paragraph{Impact} The contribution of this work is primarily theoretical and not confined to a specific domain. Its potential societal impact would, therefore, be closely tied to the particular scenarios to which it is applied.
It could be interesting to compare the environmental impact of finding lottery tickets inside overparameterized networks. We also believe that our work has the potential to have a strong environmental impact as sparse NNs have  massively reduced inference costs.

\newpage
\begin{ack}
This research is supported by the EPSRC grant EP/W005573/1, and by the France 2030 program, managed by the French National Research Agency under grant agreements No. ANR-23-PECL-0003 and and ANR-22-PEFT-0002. It was also funded in part by the European Network of Excellence dAIEDGE under Grant Agreement Nr. 101120726, by SmartNet and LearnNet, and by the French government National Research Agency (ANR) through the UCA JEDI (ANR-15-IDEX-01), EUR DS4H (ANR-17-EURE-004), and the 3IA Côte d’Azur Investments in the Future project with the reference number ANR-19-P3IA-0002.  
\end{ack}

\printbibliography
\newpage
\appendix

\section{Lower Bound on the Ticket Size in \cite{Pensia}}
\label{apx:pensia_lower}

% The claim is a direct consequence of \textit{Step 2-3} in Appendix \ref{apx:lower}. There, we will essentially show that, for linear functions, if a network having $n$ parameters admits a high probability approximation by some subnetwork of size $k$, we must have
% \[
%     {n \choose k} \geq \frac 12 \left(\frac 1{2\epsilon}\right)^{d^2},
% \]
% as ${n \choose k}$ is the number of all possible subnetworks of size $k$ that can be formed, and, in the notation of Appendix \ref{apx:lower}, $n$ being the number of parameters of the original network ($m$ in the Introduction).
% The above, together with the elementary bound
% \[
%      {n \choose k} \leq 2^{k\log_2 (\frac nk e)}, 
% \]
% yields 
% \[
%     k\log \left(\frac nk e\right) \geq  d^2 \log \left(\frac 1{2\epsilon}\right) -1 
% \]
% or equivalently,
% \[ 
%  n \geq k e^{\frac 1k (d^2 \log (\frac 1{2\epsilon}) -1)-1}. 
% \]
% The last equation implies that $k = \Theta(n)$ is necessary, in order to get $n = \Theta ( \log (\frac 1{2\epsilon}))$. To see why, suppose that there exists a subsequence such that $k_j = o(n_j)$ as $\epsilon_j$ vanishes. Then the inequality obtained reads \[n_j\ge o(n_j)2^{\frac{n_j}{o(n_j)}}\ge 2^{\frac{n_j}{o(n_j)}}.\] This is equivalent to
% \[\log_2 \frac{n_j}{o(n_j)}\ge \frac{n_j}{o(n_j)},\] which is clearly a contradiction. Therefore, no such subsequence exists, meaning that as $\epsilon$ vanishes, $k=\Omega(n)$, which implies the claim.

The claim is a direct consequence of the proof of \cite[Theorem 2]{Pensia} (Appendix B). 
There, in Step 3, it is shown that 
\[
    |\mathcal G| \geq \frac 12 \left(\frac 1{2\epsilon}\right)^{d^2},
\]
where $\mathcal G$ is the set of subnetworks that can be formed. 
Let $m$ be the number of parameters of the original network.
If we consider subnetworks of size at most $\gamma m$ ($0 \le \gamma \le 1$), we have\footnote{follows from the upper bound $\sum_{i=1}^{k} {n \choose i} \leq \left(\frac{en}{k}\right)^k$ on the partial sum of binomial coefficients.} 
\[
    |\mathcal G| \leq \sum_{i=1}^{\gamma m} {m \choose {\gamma m}} \leq 2^{\gamma m\log_2 (\frac{m}{\gamma m} e)}, 
\]
which combined with the previous inequality implies 
\[
    \gamma m \log_2 \left(\frac{e}{\gamma}\right) \geq  d^2 \log_2 \left(\frac 1{2\epsilon}\right) -1 
\]
If we have an overparameterized network of size \( m = \mathcal{O}(d^2 \log_2\left(\frac{1}{2\epsilon}\right)) \), as in \cite{Pensia}, we need \( \gamma m = \Theta(m) \) for the last inequality to be satisfied (note that $\log_2\left(\frac{e}{\gamma}\right) \le 1$, as $0 \le \gamma \le 1$).

\section{Visualizations}

\begin{figure}[ht!]
\centering
\includegraphics[width=\textwidth]{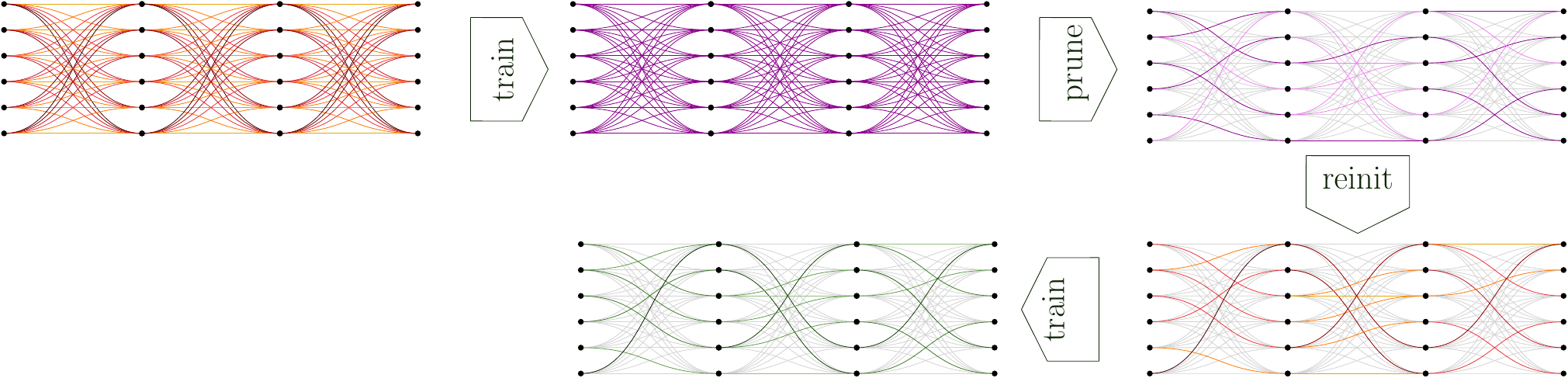}
\caption{\textbf{Simplified representation of the procedure for finding Lottery Tickets (LTH)}. A large random neural network (step 1) is trained by iterative  pruning with rewind: when the loss reaches a local minimum (step 2), some weights with smallest absolute value are pruned (step 3) and the value of the remaining edges is then reset to that of the initialization (step 4); finally, training is resumed and the final network is obtained (step 5). Remarkably, the sparser subnetwork is consistently able to reach a loss not larger than that right after pruning.}\label{fig:LTH}
\end{figure}

\begin{figure}[ht!]
\centering
\includegraphics[page=2,width=0.7\textwidth,trim={0 5cm  5cm 0},clip]{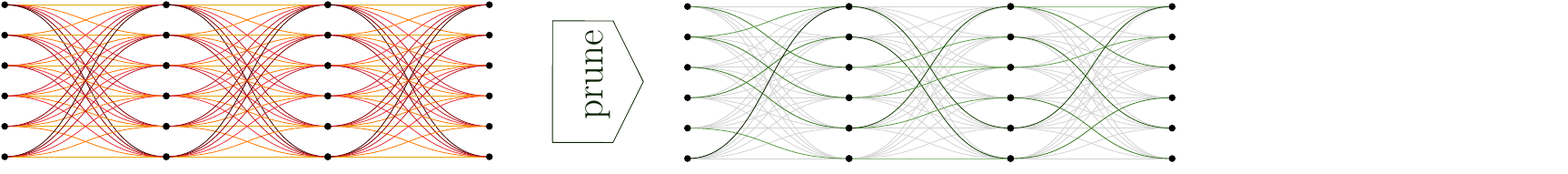}
\caption{
\textbf{Simplified representation of the procedure for finding Strongly Lottery Tickets (SLTH) / Training by pruning}. Previous work has shown that it is possible to sparsify large random neural network in order to obtain subnetworks that achieve good performance for a task under consideration, motivating the \emph{Strong Lottery Ticket Hypothesis}. No training is required.}\label{fig:truning}
\end{figure}

\section{Proof of Uniform$[-1,1]$ being \AlmostGaussian{}}
\label{apx:uniform_has_property}

In this section we provide a detailed proof of Lemma \ref{lem:uniform}, which states that the uniform distribution in $[-1,1]$ is \almostGaussian{}, as stated in Definition~\ref{def:quasi_unif}.
We remark that, while the proof is written for uniform random variables, it should be possible to extend it to a %larger 
family of densities which are unimodal, with bounded variance, and bounded third moment.

\begin{proof}[Proof of Lemma \ref{lem:uniform}]
    Note first that the distribution of the sum of $n$ i.i.d. variables in $[0,1]$ is known as the Irwin–Hall distribution $\Irv$.\footnote{It should be known that $\Irv$ is unimodal  with a mode in $n/2$, but we were not able to find a reference.
    It is instructive to note, assuming that $\Irv$ is unimodal  with a mode in $n/2$, it directly follows that its probability density function is increasing on the interval $[0,n/2]$, and then decreasing over $[n/2,n]$. This implies that $\f(x,n)$ (the density of $\Sigma^{\mathcal U_n}_{[n]}$), is non decreasing in the interval $[-n,0]$, has maximum at $0$, and non increasing $[0,n]$ for all $n\ge 2$.} 
    We will use that $\variance[\Irv] = \frac n{12}$, where $\variance[X]$ denotes the variance of the random variable $X$. 
    
    For $n\geq 2$, $f(x,n)$ can be defined as the convolution of $f(x)=f(x,1)$ and $f(x,n-1)$, i.e.,
    \[
    f(x,n) = \int_{-\infty}^{+\infty} f(x-\tau,n-1)f(\tau) d\tau.
    \]
    It is straightforward to show, by induction and an elementary substitution in the integral above, which is relied upon in the inductive step, that $f(x,n)$ is symmetric about $0$, that is $f(x,n)=f(-x,n)$. 

    Let us now prove by induction that $\f(x,n)$ is nondecreasing on the interval $[-n,0]$ and nonincreasing over $[0,n]$ (for simplicity, since it vanishes outside $[-n,n]$, we can consider directly the negative half and positive half of the real line, respectively, in the argument that follows).
    
    The claims hold trivially for $f(x)$; also note that
    \[
    f(\tau) = \left\{ 
    \begin{array}{ll}
    \frac{1}{2} & \qquad \text{if} \quad -1 \leq \tau \leq 1\\
    0 & \qquad \text{otherwise}
    \end{array}\right. \qquad {\implies} \qquad 
    f(x,n) = \frac{1}{2}\int_{-1}^{+1} f(x-\tau,n-1) d\tau.
    \]

    If $x \leq x' \leq - 1$. Since $x-\tau\le x'-\tau\le 0$, by inductive hypothesis we have that $f(x-\tau,n-1) \leq f(x'-\tau,n-1)$ over the whole interval $\tau \in [-1,1]$. Taking integrals yields $f(x,n)\leq f(x',n)$. 
    %Similarly, if $1\leq x \leq x'$, we have $f(x-\tau,n-1) \geq f(x'-\tau,n-1)$ for $\tau \in [-1,1]$, and $f(x,n)\geq f(x',n)$.
    
    Now, consider the case when $x \leq -1 \leq x' \leq 0$. 
    If $x+1 \leq -x'-1$, $x-\tau\le x+1\le-x'-1\le-x'+\tau\le -x+\tau$. By the symmetry about the origin, the inductive hypothesis is $f(x-\tau,n-1)=f(-x+\tau,n-1) \leq f(-x'+\tau,n-1)=f(x'-\tau,n-1)$ over the whole interval $\tau \in [-1,1]$, since $-1\le-x'+\tau\le-x+\tau$. Taking integrals yields $f(x,n)\leq f(x',n)$.
    Otherwise, there exists $\tau_0$ such that $x-\tau_0=-x'-1$, $x-\tau>-x'-1$ for all $\tau\in[-1,\tau_0)$ and $x-\tau<-x'-1$ for all $\tau\in(\tau_0,1]$. By symmetry, using $-x=x'+1-\tau_0$,
    $f(x-\tau,n-1) =f(-x+\tau,n-1)= f(x'+1+\tau-\tau_0,n-1)$. Thus, for all $\tau\in[-1,\tau_0]$, via the change of variable $\sigma=-(1+\tau-\tau_0)$ in the middle integral below, we obtain that
    \begin{equation}\label{eq:first_half_int}
        \int_{-1}^{\tau_0} f(x-\tau,n-1) d\tau = \int_{-1}^{\tau_0} f(x'+1+\tau-\tau_0,n-1) d\tau=\int_{-1}^{\tau_0} f(x'-\sigma,n-1) d\sigma.
    \end{equation}
    For all $\tau\in (\tau_0,1]$, $x-\tau<-x'-1\le-x'+\tau\le-x+\tau$, by symmetry about the origin we have that $f(x-\tau,n-1) \leq f(x'-\tau,n-1)$ by the inductive hypothesis with the same reasoning of the case $x+1 \leq -x'-1$. Taking integrals over the range $[\tau_0,1]$ for each term of the inductive hypothesis yields 
    \begin{equation}\label{eq:second_half_int}
    \int_{\tau_0}^1f(x-\tau,n-1)d\tau\le \int_{\tau_0}^1f(x'-\tau,n-1)d\tau
    \end{equation}
    Eqs. \ref{eq:first_half_int} and \ref{eq:second_half_int} imply that $f(x,n) \leq f(x'n)$. 
    
    Trivially, if $ -1 \leq x \leq x'\le 0$, analogous ideas are put in place as for the previous case, therefore we omit the details. We have thus shown the nondecreasing monotonicity of $f(x,n)$ on the negative half of the real line. By the symmetry of $f(x,n)$ about the origin, on the positive half of the real line the nondecreasing monotonicity turns into nonincreasing monotonicity, and the proof is complete.
    
    \nitbf{Lower bound (first inequality in Eq. \ref{eq:uniform_property}).} 
    The variance of $\Sigma^{\mathcal U_n}_{[n]}$ is $n/3$ since $\Sigma^{\mathcal U_n}_{[n]} = 2(\Irv(n)-n/2)$ and $\variance[\Irv(n)] = n/12$.  
    We define $\Zun = \frac{\Sun}{\sqrt{n/3}}$ and we note with $\Fn$ its cumulative distribution function. $\Zun$ has expectation 0 and standard deviation 1. 
    Consider the probability 
    \[
        \PL = \Pr(\sqrt{n} \leq \Sun \leq 2\sqrt{n}) 
        = \Pr(\sqrt 3 \leq \Zun \leq 2\sqrt 3).
    \]
    Now, we use the following form of Berry–Esseen inequality, discussed in \cite{marengo2017geometric}[p.2]).\footnote{It is also possible to obtain our result via classical Berry-Esseen inequality, due to the improved upper bound of $0.4748$ on the absolute constant, provided in \cite{Shevtsova11}. This would require replacing with $900$ the cut-off value for $n$, which is $18$ in the current version of the argument.}
    \begin{theorem}[Allasia~\cite{Allasia81}] 
    \label{thm:BE}
    For all $n \geq 1$, 
        \[|\Fn(z)-\Phi(z)| \leq \frac{\sqrt{3}}{20\sqrt{n}}, \]
    where $\Phi(z)$ is the cumulative distribution function of the standard normal distribution. 
    \end{theorem}
    Theorem \ref{thm:BE} implies
    \[
    \PL \geq \Phi( 2\sqrt 3) - \Phi(\sqrt 3) - 2\cdot \frac{\sqrt{3}}{20\sqrt{n}}.
    \]
    When $n \geq 18$, 
    \[\Phi( 2\sqrt 3) - \Phi(\sqrt 3) - 2\cdot \frac{\sqrt{3}}{20\sqrt{n}} \geq \Phi( 2\sqrt 3) - \Phi(\sqrt 3) - 2\cdot \frac{\sqrt{3}}{20\sqrt{18}} = C_{18}> 0.
    \]
    That is 
    $\PL \geq C_{18} > 0$.
    When $2 \leq n < 18$, $\PL = \Fn(2\sqrt 3) - \Fn(\sqrt 3) = c_n > 0$. 
    We thus have
    \[
    \PL \geq \min\{C_i, \text{ for } 2 \leq i \leq 18\} = \Cminp > 0. 
    \]
    Recall that 
     $\PL = \Pr(\sqrt{n} \leq \Sun \leq 2\sqrt{n})$.
    As the density $\f(x,n)$ is decreasing on $\mathbb R^+$, we have 
    \[
    \PL \leq \f(\sqrt{n},n) \sqrt n. 
    \]
    Thus, 
    \[
    \f(\sqrt{n},n) \geq \frac{\PL}{\sqrt n}.  
    \]
    Since $\PL \geq \Cminp$ then for all $n \geq 2$
    \[
    \f(\sqrt{n},n) \geq \frac{\Cminp}{\sqrt n}. 
    \]
    When $n=1$, the density $f(1,1)=\frac 12$. So, by setting $\Cmin = \min(\Cminp,\frac 12)$, we get that, for all $n \geq 1$, for all $0\le x\le \sqrt{n}$:
    \[
    \f(x,n)\geq \f(\sqrt{n},n) \geq \frac{\Cmin}{\sqrt n}. 
    \]
    By a symmetric argument, we also have for all $ n \geq 1$, for all $-\sqrt{n}\le x\le 0$:
    \[
    \f(x,n)\geq \f(-\sqrt{n},n) \geq \frac{\Cmin}{\sqrt n}.  
    \]
    
    \nitbf{Upper bound (second inequality in Eq. \ref{eq:uniform_property}).} 
    Here, we bound the probability distribution function $\f(x,n)$ of $\Sun=\sqrt{\nicefrac{n}{3}} Z_n$, where we recall that $\Zun = \frac{\Sun}{\sqrt{\nicefrac{n}{3}}}$. Denoting $\fz$ the probability distribution function of $\Zun$, we have
    \[
    \fz(x,n) = \f\left(\sqrt{\frac{n}{3}}x,n\right)\sqrt{\frac{n}{3}}.
    \]
    We use the following local limit theorem, discussed in \cite{petrov1975sums}[p.214].
    \begin{theorem}[Sahaidarova~\cite{shakhaidarova1966uniform}]
    Let $\{X_n\}$ be a sequence of independent random variables with a
    common density $p(x)$, such that
    $E[|X_1|^3]<\infty$, $E[X_1]= 0$, $E[X^2_1]= 1$ and $\sup p(x) \leq C$.
    Let $p_n(x)$ be the density of the random variable $\frac 1{\sqrt{n}} \sum_{j=1}^n X_j$. Then
    \[
    \sup_x |p_n(x) - \phi(x)| \leq \frac{A\beta_3}{\sqrt{n}} \max(1,C^3),
    \]
    % \[
    % \sup_x |p_n(x) - \frac{1}{\sqrt{2\pi}}{e^{-x^2/2}}| \leq \frac{A\beta_3}{\sqrt{n}} \max(1,C^3),
    % \]
    where $\phi$ is the probability distribution function of a standard gaussian, $A$ is an absolute constant, and $\beta_3 = E[|X_1|^3]$. \futurenote{Why not simply replace beta3?}
    \end{theorem}
    
    The theorem can be applied to a uniform continuous distribution with density $p^u(x) = \frac{1}{2\sqrt 3}$ in the interval $[-\sqrt{3},\sqrt{3}]$, which has mean $0$ and variance $1$. 
    We thus get, for every $x\in\mathbb{R}$,
    % \[
    % p_{\indexn}^u(x) \leq \frac{1}{\sqrt{2\pi}}{e^{-x^2/2}} + \frac{A\beta_3}{\sqrt{\indexn}} \max(1,C^3) \leq c_2 + \frac{c'}{\sqrt{\indexn}} \leq c_4
    % \]
    \[
        \fz(x,n)=p_{\indexn}^u(x) \leq \phi(0) + \frac{A\beta_3}{\sqrt{\indexn}} = \frac{1}{2\pi} + \frac{A\beta_3}{\sqrt{\indexn}} \leq \frac{1}{2\pi} + A\frac{3\sqrt 3}{4} = \Cmaxp.
    \]
    In conclusion, setting $\Cmax=\sqrt{3}\Cmaxp$, for every $x\in\mathbb{R}$ it holds that
    \[
        \f(x,n)=\sqrt{\frac{3}{n}}\fz\left(\sqrt{\frac{3}{n}}x,n\right) \leq \frac{\sqrt 3\Cmaxp}{\sqrt{n}} = \frac{\Cmax}{\sqrt{n}}.
    \] 
\end{proof}

\section{Proof of Corollary \ref{cor:SRSS_amp}}\label{apx:amplif}
\begin{proof}[Proof of   Corollary \ref{cor:SRSS_amp}]
As anticipated, we proceed in three steps.
\paragraph{Step 1: Hoeffding bound.} We start by showing, following the idea at the base of \cite[Corollary 3.3]{Lueker98}, that if $n'$ is large enough, a standard Hoeffding bound ensures that with high probability a constant fraction of the sample follows a Uniform$[-1,1]$ distribution. Since we assumed that every $X_i$ is a mixture of a Uniform$[-1,1]$ distribution with probability $p$, and another distribution with density $g$ (given by the factors $G_i$), we can rewrite $X_{i}=B_{i}\cdot U_{i}+(1-B_{i})\cdot G_{i}$, with $U_i$ being the uniform random variable, $G_i$ being the random variable with density $g$, $B_i$ being independent Bernoulli random variables with probability $p$. 

Fix $\alpha=\alpha(p)\neq p$, and assume, for now, that $n'$ satisfies Eq. \ref{eq:bound_n}, and therefore, since $\epsilon<\nicefrac{1}{2}$, choosing $c_{\text{hyp}}=c_{\text{hyp}}(p)\ge (\alpha-p)^{-2}$, ensures that, defining $\epsilon'=\nicefrac{\epsilon}{2}$, 
\[n'\ge c_{\text{hyp}}\log_2\frac{1}{\epsilon}\ge \frac{1}{2(\alpha-p)^{2}}\ln\frac{1}{\epsilon'}\] 
and therefore
\[
    \Pr\left(\sum_{i}^{n'}B_{i} \leq \alpha n'\right)\leq e^{-2(\alpha-p)^2n'}\le e^{-\ln\frac{1}{\epsilon'}} =\epsilon'.
\]
Thus 
\[\Pr\left(\sum_{i}^{n'}B_{i}>\alpha n'\right)\geq 1-\epsilon',\] that is, with high probability, there is a set of indices $I\subseteq\left[n'\right]$ of size $\left|I\right|\geq\alpha n'$, such that for each $i\in I$ it holds $B_{i}=1$, i.e. $X_{i}$ is uniformly distributed.

\paragraph{Step 2: Application of Theorem~\ref{thm:srss} via rejection-sampling.} Lemma \ref{lem:uniform} ensures that the uniform distribution of the $\left|I\right|$ random variables selected in \textit{Step 1} is \almostGaussian{}. Conditionally on the event $\{\sum_{i}^{n'}B_{i}>\alpha n'\}$, we can discard all random variables indexed outside $I$ and apply directly Theorem~\ref{thm:srss} to $\alpha n'$ of the remaining ones, for any fixed $k$ and $z\in[-\sqrt{k},\sqrt{k}]$, since $\alpha c_{\text{hyp}}\ge 1$ by construction. This guarantees a success probability of $c'_{\text{thm}}$ for approximating the given target $z$; thus, 
    \begin{align*} 
    &\Pr\left(\exists S_{z}\subset\left[n\right],\left|S_{z}\right|=k:\left|\Sigma_{S_{z}}-z\right|<\epsilon'\right) \ge \\&\Pr\left(\exists S_{z}\subset\left[n\right],\left|S_{z}\right|=k:\left|\Sigma_{S_{z}}-z\right|<\epsilon'\bigg\vert \sum_{i}^{n'}B_{i}>\alpha n'\right)\Pr\left(\sum_{i}^{n'}B_{i}>\alpha n'\right) \ge\\& \Pr\left(\exists S_{z}\subset I,\left|S_{z}\right|=k:\left|\Sigma_{S_{z}}-z\right|<\epsilon'\bigg\vert |I|>\alpha n'\right)(1-\epsilon')\ge c'_{\text{thm}}(1-\epsilon')\ge \frac{3}{4}c'_{\text{thm}}=c_{\text{thm}}.
    \end{align*}

\paragraph{Step 3: Amplification.} Finally, by a standard probability amplification argument and a union bound applied to Theorem \ref{thm:srss}, by paying an extra factor $\log_2(k/\epsilon)$ in Eq. \ref{eq:bound_n}, the constant $c_{\text{thm}}$ can be amplified to $1-\epsilon$, and the existence of a suitable subset $S_z$ holds simultaneously for all $z\in\left[-\sqrt{k},\sqrt{k}\right]$. We now give more details on this amplification.

Recall that $\epsilon' = \frac{\epsilon}{2}$, and let $c_{\text{amp}} =c_{\text{amp}}(p)= 8 \frac{c_{\text{hyp}}}{c_{\text{thm}}}$ and   $ r = \frac{4}{c_{\text{thm}}}\ln \frac{k}{\epsilon} $.
By assumption,  \[ n\geq 
c_{\text{amp}}\frac{\log_{2}^2\frac{k}{\varepsilon}}{H_{2}\left(\frac{k}{n}\right)} \geq 
2r c_{\text{hyp}}\frac{\log_{2}\frac{k}{\varepsilon}}{H_{2}\left(\frac{k}{n}\right)}
\geq r c_{\text{hyp}}\frac{\log_{2}\frac{k}{\varepsilon'}}{H_{2}\left(\frac{k}{n}\right)},\]
where the last inequality is ensured by $\epsilon<\nicefrac{1}{2}$. By \textit{Step 2}, we can apply Theorem~\ref{thm:srss}, with $\epsilon'$ and $n' \ge  c_{\text{hyp}}\frac{\log_{2}\frac{k}{\varepsilon'}}{H_{2}\left(\frac{k}{n}\right)} = n^*$, allowing us to prove that we can $\epsilon'$-approximate any  target $z$ with probability at  least $c_{\text{thm}}$. The probability of failing to approximate some given $z$ is then at most $1 - c_{\text{thm}}$. From the sample $\Omega$ of \almostGaussian{} random variables take $r$ subsamples (without replacement) of cardinality $n^*$ each, $\Omega_1,\ldots,\Omega_r$. The probability of failing to approximate some given $z$ with subsetsums from $\Omega$ is less than that of failing to approximate it with subsetsums from within every $\Omega_i$'s, and the latter probability is at most $(1-c_{\text{thm}})^r$; thus, for every $z\in[-\sqrt{k},\sqrt{k}]$,
    \[ \Pr\left(\nexists S_{z}\subset\left[n\right],\left|S_{z}\right|=k:\left|\Sigma_{S_{z}}-z\right|<\epsilon'\right) \leq (1-c_{\text{thm}})^r .\] 
    By an union bound, we also have that 
    \begin{align*}
         & \Pr\left(\forall z\in\left[-\sqrt{k},\sqrt{k}\right],\exists S_{z}\subset\left[n\right],\left|S_{z}\right|=k:\left|\Sigma_{S_{z}}-z\right|<\epsilon\right) \\
        & \ge \Pr\left(\forall z\in\left\{ -\sqrt{k}+i{\epsilon'}:i\in\left[\frac{2}{\epsilon'}\sqrt{k}\right]\right\},\exists S_{z}\subset\left[n\right],\left|S_{z}\right|=k:\left|\Sigma_{S_{z}}-z\right|<\epsilon',\right) \\
         & =1-\Pr\left( \exists z\in\left\{ -\sqrt{k}+i{\epsilon'}:i\in\left[\frac{2}{\epsilon'}\sqrt{k}\right]\right\} ,
         \nexists S_{z}\subset\left[n\right],\left|S_{z}\right|=k:\left|\Sigma_{S_{z}}-z\right|<\epsilon'\right) \\
         & \geq1-\sum_{z\in\left\{ -\sqrt{k}+i{\epsilon'}:i\in\left[\frac{2}{\epsilon'}\sqrt{k}\right]\right\} }\Pr\left(\nexists S_{z}\subset\left[n\right],\left|S_{z}\right|=k:\left|\Sigma_{S_{z}}-z\right|<\epsilon'\right) \\
         & \geq 1-\frac{2}{\epsilon'}\sqrt{k}\left(1-c_{\text{thm}}\right)^{r}=1-\frac{2}{\epsilon'}\sqrt{k}\exp\left(\frac{4}{c_{\text{thm}}}\ln \left(\frac{k}{\epsilon}\right)\cdot\ln(1-c_{\text{thm}})\right)
         \\
         &\geq 1-  \frac{2}{\epsilon'}\sqrt{k} \exp\left( - 4\ln \frac{k}{\epsilon} \right) = 1- \frac{2}{\epsilon'}\sqrt{k} \frac{\epsilon^4}{k^4} \geq 1-4\epsilon^3\ge 1-\epsilon,
    \end{align*}
    where the last inequality is ensured by $\epsilon<\nicefrac{1}{2}$. This completes the proof.
%     It remains to find an $r$ such that $1-\frac{2}{\epsilon'}\sqrt{k}\left(1-c_{\text{thm}}\right)^{r}\geq 1-\epsilon$ 
%     in Eq. \ref{eq:find_r}.
%      Since $ r = 4c_{\text{thm}}\log \frac{k}{\epsilon} $
%     We have
%     \[
%         \left(1-c_{\text{thm}}\right)^{r}\leq 
      
% \leq
%         \frac{ \epsilon^4}{k^4}\leq
%         \frac{\epsilon \cdot \epsilon'}{2\sqrt{k}}.
%     \]
%    Hence, we have
%    $1-\frac{2}{\epsilon'}\sqrt{k}\left(1-c_{\text{thm}}\right)^{r}\geq 1-\epsilon$.
    
%     Overall, since we considered $r$ sets of $n^*$ variables, we a used a total of
%     \begin{equation}
%     \label{eq:calc_lb_n_amp}
%         r \cdot n^* \leq  c_{\text{amp}}\frac{\log_{2}^2\frac{k}{\varepsilon}}{H_{2}\left(\frac{k}{n}\right)}
%     \end{equation}
%     random variables, where 
%     %
%     Eq.~\ref{eq:calc_lb_n_amp} thus gives then the lower bound for $n$ stated in Corollary \ref{cor:SRSS_amp}.
\end{proof}

\section{Proof of Corollary \ref{cor:SRSS_amp_simp}}\label{apx:amplif_simp}
   \begin{proof}[Proof of   Corollary \ref{cor:SRSS_amp_simp}]
    By definition of binary entropy, we have
    \begin{equation}
        H_{2}\left(\frac{k}{n}\right) = \frac{k}{n} \log_2\left(\frac{n}{k}\right) + \left(1-\frac{k}{n}\right) \log_2 \frac{n}{n-k}
    \end{equation}
    % In particular, since $k \le \nicefrac{1}{2}$ by assumption, the first term of \cref{eq:edit_bound_n} is always bigger than the second, hence 
    In particular, since both terms in the previous equation are positive, we get
    \begin{equation}
        H_{2}\left(\frac{k}{n}\right) \ge \frac{k}{n} \log_2\left(\frac{n}{k}\right) \label{eq:h_lb_v2}
    \end{equation}
    We now use \cref{eq:h_lb_v2} to derive an upper bound for the quantity $\frac{c_{\text{amp}}}{H_{2}\left(\frac{k}{n}\right)} \frac{\log_2^2{k}+2log_2{k} \cdot 
log_2{\nicefrac{1}{\varepsilon}}}{n}$, which will be used later:
    \begin{align}
    \frac{c_{\text{amp}}}{H_{2}\left(\frac{k}{n}\right)} \frac{\log_2^2{k}+2log_2{k} \cdot 
log_2{\frac{1}{\varepsilon}}}{n} &\le \frac{c_{\text{amp}}}{\frac{k}{n} \log_2\left(\frac{n}{k}\right)} \frac{\log_2^2{k}+2log_2{k} \cdot  log_2{\frac{1}{\varepsilon}}}{n} \nonumber \\
    &= c_{\text{amp}} \frac{\log_2^2{k}+2log_2{k} \cdot  log_2{\frac{1}{\varepsilon}}}{k} \frac{1}{\log_2\left(\frac{n}{k}\right)} \label{eq:step_k_bounded_v2} \\
    &\le c_{\text{amp}} \frac{\log_2^2{k}+2log_2{k} \cdot  log_2{\frac{1}{\varepsilon}}}{k} \label{eq:step_k_bounded2_v2} \\
    &\le \frac{1}{2}, \label{eq:ub1_v2}
    \end{align}
    where from \cref{eq:step_k_bounded_v2} to \cref{eq:step_k_bounded2_v2} we used that $log_2{\nicefrac{n}{k}} \ge 1$ for $k \le \nicefrac{n}{2}$, and then the hypothesis $k \geq 2 c_{\text{amp}} \left(\log_{2}^2 k + 2log_2{k} \cdot  log_2{\frac{1}{\varepsilon}}\right)$ directly gives \cref{eq:ub1_v2}.
    Let us now rewrite \cref{eq:n_cond_amp} in a more convenient form:
    \begin{align}
        &n \frac{H_{2}\left(\frac{k}{n}\right)}{c_{\text{amp}}} \ge \log_{2}^2\frac{k}{\varepsilon} \nonumber \\
        &n \frac{H_{2}\left(\frac{k}{n}\right)}{c_{\text{amp}}} \ge \log_{2}^2 k + 2log_2{k} \cdot  log_2{\frac{1}{\varepsilon}} + \log_{2}^2\frac{1}{\varepsilon} \nonumber \\
        &n \left(\frac{H_{2}\left(\frac{k}{n}\right)}{c_{\text{amp}}} - \frac{\log_{2}^2 k + 2log_2{k} \cdot  log_2{\frac{1}{\varepsilon}}}{n} \right) \ge \log_{2}^2\frac{1}{\varepsilon} \nonumber \\
        &n \left(1 - \frac{c_{\text{amp}}}{H_{2}\left(\frac{k}{n}\right)} \frac{\log_{2}^2 k + 2log_2{k} \cdot  log_2{\frac{1}{\varepsilon}}}{n} \right) \ge  \frac{c_{\text{amp}}}{H_{2}\left(\frac{k}{n}\right)} \log_{2}^2\frac{1}{\varepsilon} \\
        &n \ge \frac{c_{\text{amp}}}{\left(1 - \frac{c_{\text{amp}}}{H_{2}\left(\frac{k}{n}\right)} \frac{\log_{2}^2 k + 2log_2{k} \cdot  log_2{\frac{1}{\varepsilon}}}{n} \right)} \frac{\log_{2}^2\frac{1}{\varepsilon}}{H_{2}\left(\frac{k}{n}\right)} \label{eq:edit_bound_n_v2} 
    \end{align}
    Using \cref{eq:ub1_v2} we get
    \begin{align}
    \frac{c_{\text{amp}}}{\left(1 - \frac{c_{\text{amp}}}{H_{2}\left(\frac{k}{n}\right)} \frac{\log_2^2{k}+2log_2{k} \cdot 
log_2{\frac{1}{\varepsilon}}}{n} \right)} \le 2 c_{\text{amp}}
    \end{align}
    To satisfy \cref{eq:edit_bound_n_v2}, we can then choose $n$ such that
    \begin{align}
    n &\ge 2 c_{\text{amp}} \frac{\log_{2}^2\frac{1}{\varepsilon}}{H_{2}\left(\frac{k}{n}\right)},
    \end{align}
    and then apply \Cref{cor:SRSS_amp} to end the proof.
    \end{proof}

\section{Details for the proof of Theorem \ref{thm:srss}}
\label{apx:srss_proof_details}

\paragraph{Proof of Eq. \ref{eq:second_factor_second_term}}
Define $A=\tilde{S}^{\prime}\backslash\tilde{S}$,
and observe that
\begin{align}
 & \Pr\left(\Hz{\tilde{S}^{\prime}}\,|\,\Hz{\tilde{S}},I_{\mu k,k-1}\right)\label{eq:start_of_law_of_tot_prob}\\
 & =\sum_{i=\mu k}^{k-1}\Pr\left(\Hz{\tilde{S}^{\prime}}\,|\,\Hz{\tilde{S}},I_{i}\right)\Pr\left(I_{i}\,|\,\Hz{\tilde{S}},I_{\mu k,k-1}\right)\label{eq:first_law_of_tot_prob}\\
 & =\sum_{i=\mu k}^{k-1}\int_{-\infty}^{\infty}\Pr\left(\left|\Sigma_{A}-\left(z-y\right)\right|<\epsilon\,|\,\Sigma_{I}=y,I_{i},\Hz{\tilde{S}}\right)\Pr\left(\Sigma_{I}=y\,|\,\Hz{\tilde{S}},I_{i}\right)dy\notag\\
 & \qquad\cdot\Pr\left(I_{i}\,|\,\Hz{\tilde{S}},I_{\mu k,k-1}\right)\label{eq:second_law_of_tot_prob}\\
 & =\sum_{i=\mu k}^{k-1}\int_{-\infty}^{\infty}\Pr\left(\left|\Sigma_{A}-\left(z-y\right)\right|<\epsilon\,|\,\Sigma_{I}=y, I_{i}\right)\Pr\left(\Sigma_{I}=y\,|\,\Hz{\tilde{S}},I_{i}\right)dy\notag\\
 & \qquad\cdot\Pr\left(I_{i}\,|\,\Hz{\tilde{S}},I_{\mu k,k-1}\right)\label{eq:after_dropping_conditioning}\\
 & \leq c\varepsilon\sum_{i=\mu k}^{k-1}\int_{-\infty}^{\infty}\Pr\left(\Sigma_{I}=y\,|\,\Hz{\tilde{S}},I_{i}\right)dy\Pr\left(I_{i}\,|\,\Hz{\tilde{S}},I_{\mu k,k-1}\right)\label{eq:last_step}\\
 &\le c\varepsilon \nonumber
\end{align}
where from Eq. \ref{eq:start_of_law_of_tot_prob} to Eq. \ref{eq:first_law_of_tot_prob}
and from Eq. \ref{eq:first_law_of_tot_prob} to Eq. \ref{eq:second_law_of_tot_prob}
we used the law of total probability;\footnote{For simplicity, we denote the density of $\Sigma_{I}$ conditional on $\Hz{\tilde{S}}\cap I_{i}$ as $\Pr\left(\Sigma_{I}=y\,|\,\Hz{\tilde{S}},I_{i}\right)$.} from Eq. \ref{eq:second_law_of_tot_prob}
to Eq. \ref{eq:after_dropping_conditioning} we dropped the redundant
event $\Hz{\tilde{S}}$ in the conditioning, due to conditional independence;
finally, from Eq. \ref{eq:after_dropping_conditioning} to Eq. \ref{eq:last_step}
we used Definition \ref{def:quasi_unif} which implies that for any
$i\in\left\{ \mu k,...,k-1\right\} $ it holds
\begin{align*}
 & \Pr\left(\left|\Sigma_{A}-\left(z-y\right)\right|<\epsilon\,|\,\Sigma_{I}=y,I_{i}\right)=\Pr\left(\left|\Sigma_{\left[k-i\right]}-\left(z-y\right)\right|<\epsilon\right)\leq c\varepsilon.
\end{align*}

\paragraph{Proof of Eq. \ref{eq:square_root_third_addendum}}
Let $A=\tilde{S}^{\prime}\backslash\tilde{S}$.
Analogously to the calculations from Eq.~\ref{eq:first_law_of_tot_prob} to Eq.~\ref{eq:after_dropping_conditioning}, by the law of total probability we have
\begin{align}
 & \sum_{i=0}^{\mu k-1}\Pr\left(I_{i}\right)\cdot\Pr\left(\Hz{\tilde{S}^{\prime}}\,|\,\Hz{\tilde{S}},I_{i}\right)\nonumber \\
 & =\sum_{i=0}^{\mu k-1}\Pr\left(I_{i}\right)\cdot\int_{-\infty}^{\infty}\Pr\left(\left|\Sigma_{A}-\left(z-y\right)\right|<\epsilon\,|\,\Sigma_{I}=y,\,I_{i}\right)\Pr\left(\Sigma_{I}=y\,|\,\Hz{\tilde{S}},I_{i}\right)dy\nonumber \\
 & =\sum_{i=0}^{\mu k-1}\Pr\left(I_{i}\right)\cdot\int_{-\infty}^{\infty}\Pr\left(\left|\Sigma_{\left[k-i\right]}-\left(z-y\right)\right|<\epsilon\right)\Pr\left(\Sigma_{I}=y\,|\,\Hz{\tilde{S}},I_{i}\right)dy\label{eq:using_def_in_third_addendum}\\
 & \leq c\frac{\epsilon}{\sqrt{k}}\sum_{i=0}^{\mu k-1}\Pr\left(I_{i}\right)\cdot\int_{-\infty}^{\infty}\Pr\left(\Sigma_{I}=y\,|\,\Hz{\tilde{S}},I_{i}\right)dy\label{eq:after_using_def_in_second_addendum}\\
 &\le c\frac{\epsilon}{\sqrt{k}} \nonumber
\end{align}

where from Eq. \ref{eq:using_def_in_third_addendum} to Eq. \ref{eq:after_using_def_in_second_addendum}
we used Definition \ref{def:quasi_unif}, which implies that for any
$i\in\left\{ 0,...,\frac{9}{10}k-1\right\} $ it holds
\[
\Pr\left(\left|\Sigma_{\left[k-i\right]}-\left(z-y\right)\right|<\epsilon\right)\leq c'\frac{\epsilon}{\sqrt{k-i}}\leq c\frac{\epsilon}{\sqrt{k}}.
\]

\section{Proof of Theorem \ref{thm:pensia}}
\label{apx:new_pensia}

In the proof we will refer to the following results, upon which \cite{Pensia}[Theorem 1] relies (the statement below slightly differ as we fix two small typos in their notation and mixing coefficients). With the understanding that by a mixture $D$ of a distribution $D_1$ and $D_2$ with probability $p$ it is meant that the pdf (we adopt the convention that this term includes generalised functions, such as Dirac deltas for point masses) of $D$ can be written as a convex combination of the pdf of $D_1$ and that of $D_2$, that is $f_D=pf_{D_1}+(1-p)f_{D_2}$. For the unfamiliar reader, we note that in the literature this is often stated in short as $D=pD_1+(1-p)D_2$.
\begin{lemma}[{{\cite{Pensia}[Corollary 1]}}]\label{lem:unifprod}
Let $X\sim$Uniform$[0,1]$ (or $X\sim$Uniform$[-1,0]$) and $Y\sim$Uniform$[-1,1]$ be independent random variables. Let $P$ be the distribution of the $XY$ and $\delta_0$ the Dirac delta at $0$. Let $D$ be the distribution obtained as mixture of $\delta_0$ and $P$ with probability $\nicefrac{1}{2}$. Then $D$ is the mixture of a Uniform$[-\nicefrac{1}{2},\nicefrac{1}{2}]$ and some distribution $Q$ with probability $\ln(2)/4$.
\end{lemma}
\begin{corollary}[{{\cite{Pensia}[Corollary 2]}}]\label{cor:Luek2.5}
Let $X_1,\ldots,X_n$ be iid with distribution $D$ as defined in Lemma \ref{lem:unifprod}, where $n\ge C \ln(\nicefrac{2}{\epsilon})$ for some universal constant $C$. Then
\[\Pr\left( \forall\,z\in[-1,1],\:\exists\, S\subset [n]\: :|z-\sum_{i\in S}X_i|\le \epsilon\right)\ge 1-\epsilon.\]
\end{corollary}
\begin{proof}[Proof of Theorem \ref{thm:pensia}]
    The key idea is exploiting Corollary \ref{cor:SRSS_amp} at each step of the pruning strategy established in \cite{Pensia}[Theorem 1], where Corollary \ref{cor:Luek2.5} is used instead. Without loss of generality, we replace their $\min\{\epsilon,\delta\}$ with $\epsilon$. 
    For the sake of easily following the approach adopted in \cite{Pensia}, let us define $n^*(x)$ as the function
    \begin{equation}\label{n^*bis}
        n^*(x)=c_{\text{amp}}\frac{\log_{2}^2 {(kx)}}{H_{2}\left(\frac{k}{n^*(x)}\right)}
    \end{equation}
    where $k=\gammastar n^*(x)$. In the following, we use $n^*$ as short for $n^*(1/\epsilon)$, and we will only explicitely provide an argument for $n^*$ when it is different than $1/\epsilon$. For instance, in the last step of the proof, we will use $n^*(\nicefrac{2\ell d_i d_{i-1}}{\epsilon})$, which matches the definition of $n_i^*$ given in  Eq. \ref{n^*}.
    % For the sake of the following argument, consider $n^*=n^*(\nicefrac{1}{\epsilon})$ as the implicit function satisfying 
    % \begin{equation}\label{n^*bis}
    %     n^*= c_{\text{amp}}\frac{\log_{2}^2\frac{k}{\varepsilon}}{H_{2}\left(\frac{k}{n^*}\right)}
    % \end{equation}
    % where $k=\gammastar n^*$. If, in the last step of our argument, we subsequently rescale the argument as $n^*(\nicefrac{2\ell d_i d_{i-1}}{\epsilon})$, then we satisfy Eq. \ref{n^*} with $n_i=n^*(\nicefrac{2\ell d_i d_{i-1}}{\epsilon})$. 
    % Thus, in the following, we refer to the $n^*$ satisfying Eq. \ref{n^*bis}, with the understanding that it will be suitably rescaled in the last step.
    
    Consider \cite{Pensia}[Lemma 1]. When approximating a single link (that is, a weight), after the overparameterization (which creates an additional layer of width $2n^*$ in between the input and the output node) via $4n^*$ links, instead of pruning via Corollary \ref{cor:Luek2.5}, we prune via Corollary \ref{cor:SRSS_amp} twice in the second layer, that is we ensure that only $k=\gammastar n^*$ \futurenote{add the following caveat if gamma not constant:(here $\gammastar$ is not necessarily a fixed constant, as per the previous discussion)} edges yield the desired approximation, both in the edges corresponding to the positive part of the input weights and in those corresponding to the negative part. Thus we obtain at most $4k$ surviving edges, after the preprocessing step and the pruning mask is applied.
    This yields a sparsity of at least $\alpha'=1-\gammastar$.
    Note that it is because of the preprocessing step that we go from distributions Uniform$[-1,1]$ to distributions $D$, as defined in Lemma \ref{lem:unifprod}, which are shown to be a mixture with Uniform$[-1,1]$ and therefore can be also handled via Corollary \ref{cor:SRSS_amp}.
    
    Consider \cite{Pensia}[Lemma 2]. When approximating a real-valued multivariate linear function, after the overparameterization (which creates an additional layer of width $2dn^*(\nicefrac{d}{\epsilon})$ in between the $d$ input nodes and the output node) one simply iterates the ideas of the previous case $d$ times. For each input node, the overparameterization surviving the preprocessing step on the weights of the input layer is $4n^*(\nicefrac{d}{\epsilon})$. Pruning the second layer of the overparameterized link for each input via Corollary \ref{cor:SRSS_amp} with $k=\gammastar n^*(\nicefrac{d}{\epsilon})$ (again, performing this both on the edges corresponding to the positive part of the input weights and in those corresponding to the negative part), instead of exploiting Corollary \ref{cor:Luek2.5}, yields that at most $4dk$ edges survive after the pruning mask is applied. This yields a sparsity of at least $\alpha'=1-\gamma'$.
    
    Finally, consider \cite{Pensia}[Lemma 3]. When approximating a layer with input dimension $d_1$ and output dimension $d_2$, after the overparameterization (which creates an additional layer of width $2d_1n^*(\nicefrac{d_1d_2}{\epsilon})$ in between the input nodes and the output nodes) one iterates the ideas of the previous case $d_1$ times in the input layer through the same preprocessing step, and $d_2$ times in the output layer, one for each of the $d_1$ blocks created by the preprocessing (essentially the weights in the input layer are \textit{re-used} $d_2$ times). For each input node, the overparameterization surviving the preprocessing step is at most $2(d_2+1)n^*(\nicefrac{d_1d_2}{\epsilon})$. Overall, after the preprocessing step, we have at most $2d_1(d_2+1)n^*(\nicefrac{d_1d_2}{\epsilon})$ parameters. 
    We then use Corollary~\ref{cor:SRSS_amp} (with $k=\gammastar n^*(\nicefrac{d_1d_2}{\epsilon})$) to prune the number of parameters between the introduced additional layer and the $d_2$ outputs down to $2d_1 d_2\gammastar n^*(\nicefrac{d_1d_2}{\epsilon})$.
    As for the edges between the $d_1$ inputs and the additional layer, only those that reach a neuron in the additional layer, from which there is at least one outgoing edge towards the $d_2$ outputs, are used; since for each of the $d_1$ blocks of $2n^*(\nicefrac{d_1d_2}{\epsilon})$ neurons in the additional layer we only kept $2\gammastar n^*(\nicefrac{d_1d_2}{\epsilon})$ outgoing edges to each of the $d_2$ output neurons, in the worst case (all the nodes involved in the subsetsums are disjoint) we keep $2d_2 \gammastar n^*(\nicefrac{d_1d_2}{\epsilon})$ of them for each of the $d_1$ neurons. 
    Globally, we are left with a total of at most $2d_1 d_2 \gammastar n^*(\nicefrac{d_1d_2}{\epsilon})$ edges both in the input layer and in the output layer, thus a total of $4d_1d_2 \gammastar n^*(\nicefrac{d_1d_2}{\epsilon})$ edges survive the pruning. The density of the surviving edges is then less than \[\frac{4d_1d_2 \gammastar n^*(\nicefrac{d_1d_2}{\epsilon})}{2d_1^2n^*(\nicefrac{d_1d_2}{\epsilon})+2d_1d_2 n^*(\nicefrac{d_1d_2}{\epsilon})}=\frac{2d_2 \gammastar }{d_1+d_2 }=\frac{(d_1\frac{d_2}{d_1}+d_2)\gammastar}{d_1+d_2}\le \rho_1\gammastar,\]
    where $\rho_1 = \max\left\lbrace \nicefrac{d_1}{d_2},\nicefrac{d_2}{d_1}\right\rbrace$ and in the last inequality we used that $d_1\nicefrac{d_2}{d_1}+d_2\le\rho_1(d_1+d_2)$ since $\rho_1\ge 1$. This ensures a sparsity $\alpha'\ge 1-\rho_1\gammastar$.

    \cite{Pensia}[Theorem 1] consists of performing, for every $i\in[\ell]$, the previous step on layer $i$ with input dimension $d_{i-1}$ and output dimension $d_i$. The overparameterization creates an additional layer of nodes of width $2d_{i-1}n^*(\nicefrac{2\ell d_{i-1}d_i}{\epsilon})$ in between the $d_{i-1}$ input nodes and the $d_i$ output nodes. Since the construction is stacked $\ell$ times, this generates $2\ell$ layers for the overparameterized network, which will therefore have a starting number of parameters 
    \[m = \sum_{i=1}^\ell 2 d_{i-1}^2 n^*(\nicefrac{2\ell d_{i-1}d_{i}}{\epsilon}) + 2 d_{i-1}d_{i} n^*(\nicefrac{2\ell d_{i-1}d_{i}}{\epsilon}).\]
    Corollary \ref{cor:SRSS_amp} applied to each stacked overparameterized layer instead of Corollary \ref{cor:Luek2.5}  as in the previous step yields that the total number of parameters left after the pruning is 
    \[m_t \le \sum_{i=1}^\ell 4 d_{i-1}d_{i} k_i,\] 
    where $k_i=\gammastar n^*(\nicefrac{2\ell d_{i-1}d_{i}}{\epsilon})$. Recall that $\rho = \max_{i} \rho_i$, where $\rho_i=\max\{\nicefrac{d_i}{d_{i-1}}, \nicefrac{d_{i-1}}{d_i}\}\ge 1$. Recall that $\gamma=\rho\gammastar$.
    We obtain that
    \begin{align*}
    m_t  \le &\sum_{i=1}^\ell 2 d_{i-1}\frac{d_{i}}{d_{i-1}}d_{i-1} k_i + 2 d_{i-1}d_{i} k_i \\ 
    \le &\sum_{i=1}^\ell 2 d_{i-1} \rho_i d_{i-1} k_i + 2 d_{i-1}d_{i} \rho_i k_i \\ 
    \le &\rho\sum_{i=1}^\ell 2 d_{i-1}^2 k_i + 2 d_{i-1}d_{i} k_i \\ = & \rho\gammastar \sum_{i=1}^\ell 2 d_{i-1}^2 n^*(2\ell\nicefrac{ d_{i-1}d_{i}}{\epsilon}) + 2 d_{i-1}d_{i} n^*(2\ell\nicefrac{ d_{i-1}d_{i}}{\epsilon})=\gamma m
    \end{align*}
    %We now compare $m_t$ with the size $m$ of the network before doing any pruning, which is:$$m = \sum_{i=1}^\ell 2 d_{2i-1}^2 n^*(\nicefrac{2\ell d_{2i-1}d_{2i}}{\epsilon}) + 2 d_{2i-1}d_{2i} n^*(\nicefrac{2\ell d_{2i-1}d_{2i}}{\epsilon}) $$
    We then get that the density of the edges surviving the pruning is $\nicefrac{m_t}{m} \le \gamma$, which implies a sparsity of at least $\alpha = 1 - \gamma$. 

\end{proof}

% To exemplify Theorem \ref{thm:pensia}, consider $k=\gamma n^*$ for any fixed constant $\gamma\in(0,\nicefrac{1}{2})$, $\gamma=1-\alpha$.
% This simple set-up is enough to show that Theorem \ref{thm:pensia} applies to near-optimal overparameterizations. By exploiting standard bounds on the binary entropy, it is easily verified that $n^*=n^*(\nicefrac{1}{\epsilon})=\log ^r(\nicefrac{1}{\epsilon})$
% % and $k=k(\nicefrac{1}{\epsilon})=\gamma n^*$,  
% , with $k=\gamma n^*$, 
% satisfies Eq. \ref{n^*bis}, for every $\epsilon$ small enough, $r>2$ and $\gamma\in(0,\nicefrac{1}{2})$. Here, for simplicity, we omitted possible constant terms in $d$ and $\ell$ (one can also think of it as the case $d=1$ of approximation of a single link, see the proof below). As previously mentioned, Theorem \ref{thm:pensia} requires slightly larger overparameterizations than the optimal one (which depends only logarithmically on $\nicefrac{1}{\epsilon}$) described in \cite{Pensia}. However, this is so only by additional logarithmic factors: as the example shows, we can control the sparsity of any (larger than quadratic) polylogarithmic overparameterization. In the example above, we achieve any target sparsity $\alpha=1-\gamma$ in near-optimal overparameterizations. 

\section{Proof of Theorem \ref{thm:lowerBoundStrong}}
\label{apx:lower}

\begin{proof}[Proof of Theorem \ref{thm:lowerBoundStrong}]
    Consider the space $\mathcal{W}_k=\{W \in \mathbb R^{d\times d}: \|W\|\leq \sqrt k\}$, and let $\mathcal{P}_k$ be a $2\epsilon$-separated set of $\mathcal{W}_k$, i.e. a subset 
    $\mathcal{P}_k \subset \mathcal{W}_k$ such that for all distinct $ W , W' \in \mathcal{P}_k$ it holds $\|W - W'\| > 2\epsilon$. We denote $\mathcal{W}=\mathcal{W}_1$, $\mathcal{P}=\mathcal{P}_1$, and the set of all possible subnetworks of $g$ as $\mathcal{G}$ (note that this does not denote $\mathcal{G}_1$, the set of all subnetworks of size $1$).
    %\begin{equation}\label{eq:eps_sep} \|W - W'\| > 2\epsilon.\end{equation}
    
    \paragraph{Step 1: Packing argument.} In \cite{Pensia}[Theorem 2, Step 1], it is shown that any function $g'$ can only approximate at most one member of $\mathcal{P}$ for bounded input $x$ (say, $\|x\|\le 1$). In particular, this also applies to functions $g'$ representing the elements of $\mathcal{G}_k$.
    %Assume, by contradiction, that a function $g'$ can approximate two distinct members $W_1$ and $W_2$ of $\mathcal{P}$, which is to say
    %\begin{equation*}\max_{x:\|x\|_\infty \leq 1} \|g'(x)-W_1x\|_\infty \leq \epsilon \quad {and} \quad    \max_{x:\|x\|_\infty \leq 1} \|g'(x)-W_2x\|_\infty  \leq \epsilon.\end{equation*}Using a triangle inequality we would get
    %\begin{align*}2\epsilon & \ge \max_{x:\|x\|\leq 1} \|g'(x)-W_1x\|_\infty + \max_{x:\|x\|_\infty\leq 1} \|g'(x)-W_2x\|_\infty\\ & \ge \max_{x:\|x\|_\infty\leq 1} \|W_1 x - W_2 x\|_\infty = \|W_1 - W_2\|_\infty \end{align*}which is a contradiction with Eq. \ref{eq:eps_sep}, since both $W_1$ and $W_2$ belong to the $2\epsilon$-separated set $\mathcal{P}$.Hence, a single pruned network $g'$ can approximate at most one weight in $\mathcal{P}$.
    %
    \paragraph{Step 2: Relation between $|\mathcal{G}_k|$ and $|\mathcal{P}_k|$.} By \textit{Step 1}, in \cite{Pensia}[Theorem 2, Step 2] it is shown that $|\mathcal{P}| \le 2|\mathcal{G}|$, under the assumption of Eq. \ref{eq:lowerBdApprox}, with $\mathcal{G}_k$ replaced by $\mathcal{G}$). Therefore, also by \textit{Step 1}, replacing $\mathcal{P}$ with $\mathcal{P}_k$ and $\mathcal{G}$ with $\mathcal{G}_k$ in \cite{Pensia}[Theorem 2, Step 2], it holds that $|\mathcal{P}_k| \le 2|\mathcal{G}_k|$. 
    %We proceed by contradiction, showing that, if $|\mathcal{P}|>2|\mathcal{G}|$, then  there is at least one weight in $\mathcal{P}$ that cannot be approximated with probability at least $1/2$, which contradicts the theorem assumption.Let us assume then, by contradiction, that $|\mathcal{P}|>2|\mathcal{G}|$, and recall that any element of $\mathcal{G}$ can approximate at most one weight in $\mathcal{P}$.We have:
    %\begin{align*}\frac{\sum_{w \in \mathcal{P}} \mathbb I \left( \exists g' \in \mathcal G :  \max_{x:\|x\|_\infty\leq 1} \|g'(x)-Wx\|_\infty  \leq \epsilon \right) }{|\mathcal P|} &\leq  \frac{|\mathcal{G}|}{|\mathcal{P}|} < \frac{1}{2}.\end{align*}
    %where the numerator in the first expression represents the number of weights in $\mathcal{P}$ that we can approximate using networks in $\mathcal{G}$.Taking the expectation over the distribution of $g$, we get that
    %\begin{align*}\frac{\sum_{w \in \mathcal{P}} \mathbb{P} \left( \exists g' \in \mathcal G : \max_{x:\|x\|_\infty\leq 1} \|g'(x)-Wx\|_\infty \leq \epsilon \right) }{|\mathcal{P}|} &< \frac{1}{2}. \end{align*}
    %The above says that the average probability of approximating a weight in $\mathcal{P}$ is less than $1/2$. Then, there exists at least one weight $w \in \mathcal P$ that can be approximated with probability less than $1/2$, which contradicts the assumption in Eq.~\ref{eq:lowerBdApprox} of the theorem.Therefore, $2|\mathcal{G}| >  |\mathcal P| $. 
    Note that $|\mathcal{G}_k|=\binom{n}{k}$, the number of different ways in which we can select $k$ parameters out of $n$, so we actually get
    \begin{equation}
        \label{eq:lb_binom}
        \binom{n}{k} >  \frac{|\mathcal {P}_k|}{2} .
    \end{equation}
    
    \paragraph{Step 3: Lower bound on $|\mathcal{P}_k|$.} Let us now consider a $2 \epsilon$-separated set $\mathcal{P}_k^{\max}$ of maximal cardinality. In \cite{Pensia}[Theorem 2, Step 3] it is shown that \[|\mathcal{P}^{\max}|\ge\frac{\text{Vol}(\mathcal{W})}{\text{Vol}(\{W\in\mathcal{W}:\|W\|\le 2\epsilon\})}=\left(\frac{1}{2\epsilon}\right)^{d^2}.\] 
    Here $\text{Vol}$ is the Lebesgue measure in $\mathbb{R}^{d\times d}$ identified with $\mathbb{R}^{d^2}$.%Observe that $\mathcal{P_{\max}}$ is also a $2\epsilon$-net of $\mathcal{W}$, i.e. a subset $\mathcal{P_{\max}} \subset \mathcal W$ such that for any value $W \in [-\sqrt k,\sqrt k]^{d \times d}$, there is an element in $\mathcal{P_{\max}}$ that is at distance at most $2\epsilon$ from $x$ (otherwise, we could add such an uncovered point to $\mathcal{P_{\max}}$, contradicting its maximality). 
    %Observe also that any $2\epsilon$-net has cardinality at least $\left(\frac {k}{2\epsilon}\right)^{d^2}$, since every point in the net covers at most $(4\epsilon)^{d^2}$ points of $[-\sqrt k,\sqrt k]^{d\times d}$, and the latter has volume $(4\sqrt k)^{d^2}$.
    %
    By the exact same argument, replacing $\mathcal{W}$ with $\mathcal{W}_k$ and thus $\mathcal{P}^{\max}$ with $\mathcal{P}_k^{\max}$, it holds that
    \[|\mathcal{P}_k^{\max}|\ge\frac{\text{Vol}(\mathcal{W}_k)}{\text{Vol}(\{W\in\mathcal{W}_k:\|W\|\le 2\epsilon\})}=\left(\frac{\sqrt{k}}{2\epsilon}\right)^{d^2}.\] 
    Combining this fact with Eq. \ref{eq:lb_binom} applied to $\mathcal{P}_k^{\max}$ implies that 
    \begin{equation}\label{eq:lower_bound_binomial}
        \binom{n}{k} > \frac 12 \left(\frac {\sqrt k}{2\epsilon}\right)^{d^2}.
    \end{equation}

    \paragraph{Step 4: Lower bound on $n$.} Consider the standard bound found in \cite{macwilliams1977theory}
    %\[{n \choose k}\leq\left(\frac{ne}{k}\right)^{k}=2^{k\left(\log_{2}\frac{n}{k}+\log_{2}e\right)}=2^{k\log_{2}\frac{n}{k}\left(1+\frac{\log_{2}e}{\log_{2}\frac{n}{k}}\right)}\leq2^{cnH_{2}\left(\frac{k}{n}\right)}\]
    \[{n \choose k}\leq\sqrt{\frac{n}{2\pi k(n-k)}}2^{n H_2(\nicefrac{k}{n})}.\]
    %where in the last inequality we used that $k\leq \frac n2$, 
    and combine it with with Eq. \ref{eq:lower_bound_binomial}. It follows that
    \[
     2^{nH_{2}\left(\frac{k}{n}\right)}
     \geq
     \frac 12 \sqrt {\frac{2\pi k(n-k)}{n}} \left(\frac {\sqrt k}{2\epsilon}\right)^{d^2}
    \]
    and taking the logarithm of both sides yields the sought lower bound on $n$: 
    \begin{align}
        nH_{2}\left(\frac{k}{n}\right)
        & \geq \frac{1}{2}\log_2\left(\frac{2\pi k(n-k)}{n}\right)+d^2\log_{2}\frac{\sqrt{k}}{2\epsilon} -1\label{eq:lower_bound_entropy} \\
        & \geq d^2\left(\frac{1}{2}\log_{2}k+\log_{2}\frac{1}{\epsilon}-1 \right)-1\label{eq:after_using_lambda}\\
        & \geq\frac{d^2}{2}\log_{2}\frac{k}{\epsilon}\label{eq:after_dropping_1},
    \end{align}
    where from Eq. \ref{eq:lower_bound_entropy} to Eq. \ref{eq:after_using_lambda} we exploited the definition of $\lambda$, which ensures that the first term in the r.h.s. of Eq. \ref{eq:lower_bound_entropy} is nonnegative;\footnote{This term being nonnegative is equivalent to $k(1-\nicefrac{k}{n})\ge\nicefrac{1}{2\pi}$, and since $1\le k\le\lambda n$, any $\lambda\le 1-\nicefrac{1}{2\pi}$ ensures it.} from Eq. \ref{eq:after_using_lambda} to Eq. \ref{eq:after_dropping_1} we used that for all $\epsilon<\nicefrac{1}{16}$ it holds that 
    \[d^2\left(\log_{2}\frac{1}{\epsilon}-1\right)\ge 1.\]  
\end{proof}

\section{Details of comparison with Malach et al. \cite{malachProvingLotteryTicket2020}}
\label{apx:malach}
We show that $cz \ge c_{\text{amp}} \frac{\log_2^2 (c z^2)}{\log_2(z)}$ holds for a big enough constant $c$. Recall that $z=\frac{m_t}{\epsilon}$, so we can always assume $\log_2(z) \ge 1$.
We have 
\begin{align}
\log_2^2 (c z^2) &= (\log_2 (c) + 2\log_2 (z))^2 \\
&= \log_2^2 (c) + 4\log_2 (c) \log_2 (z) + 4\log_2^2 (z) \\
&\stackrel{(a)}{\le} 6(\log_2^2 (c) + \log_2^2 (z)) \\
&\stackrel{(b)}{\le} 12 \log_2^2 (c) \log_2^2 (z),
\end{align}
where in $(a)$ we used that $2ab \le a^2 + b^2$, and in $(b)$ that $a+b \le 2ab$ for $a$ and $b$ greater than 1.

We can then focus on showing that there is a big enough constant $c$ such that $cz \ge 12 c_{\text{amp}} \log_2^2 (c) \log_2^2 (z)$. We get $c \ge 12 c_{\text{amp}} \log_2^2 (c) \frac{\log_2^2 (z)}{z}$, and we have
\begin{align}
    \log_2^2 (c) \frac{\log_2^2 (z)}{z} & \le \log_2^2 (c) \\
    &\le \sqrt{c}.
\end{align}

We can then focus on $c \ge 12 c_{\text{amp}} \sqrt{c}$, which is satisfied for $c \ge 144 c_{\text{amp}}^2$.

\end{document}